\documentclass[twoside,11pt]{article}
\usepackage{jair,rawfonts}

\usepackage[utf8]{inputenc}
\usepackage[T1]{fontenc}

\usepackage[hidelinks]{hyperref}
\usepackage{multirow}
\usepackage{arydshln}
\usepackage[dvipsnames]{xcolor}
\usepackage{graphicx}

\usepackage[natbibapa]{apacite}
\usepackage{enumitem}

\usepackage{caption}
\usepackage{subcaption}

\usepackage{algorithm2e}

\usepackage{amsmath,amsthm,amssymb}
\usepackage{bm}
\usepackage{bbm}

\usepackage{tikz}
\usepackage{pgfplots}
\usetikzlibrary{positioning, calc, fit, shapes, patterns, matrix}
\pgfplotsset{compat=1.14}

\pgfdeclarelayer{background}
\pgfdeclarelayer{foreground}
\pgfsetlayers{background,main,foreground}

\definecolor{obstaclegray}{gray}{0.665}
\definecolor{treegreen}{RGB}{33, 138, 33}
\definecolor{firered}{RGB}{255, 69, 0}
\definecolor{emitterpink}{RGB}{249, 128, 113}

\newtheorem{theorem}{Theorem}
\newtheorem{property}{Property}

\theoremstyle{definition}
\newtheorem{definition}{Definition}

\newcommand{\sublat}[2]{{\cal L} \left[ #1, #2 \right]}
\newcommand{\sublatx}[2]{{\cal L} \left[ #1, #2 \right[}

\begin{document}

    \title{Learning Pretopological Spaces to Model Complex Propagation Phenomena: A  Multiple Instance Learning Approach Based on a Logical Modeling}

    \author{
        \name Gaëtan Caillaut \email gaetan.caillaut@univ-orleans.fr \\
        \name Guillaume Cleuziou \email guillaume.cleuziou@univ-orleans.fr \\
        \addr Université d’Orléans, INSA Centre Val de Loire, LIFO EA 4022, France
    }

    \maketitle

    \begin{abstract}
        This paper addresses the problem of learning the concept of "propagation" in the pretopology theoretical formalism. Our proposal is first to define the pseudo-closure operator (modeling the propagation concept) as a logical combination of neighborhoods.
        We show that learning such an operator lapses into the Multiple Instance (MI) framework, 
        where the learning process is performed on bags of instances instead of individual instances. Though this framework is well suited for this task, its use for learning a pretopological space leads to a set of bags exponential in size. To overcome this issue we thus propose a learning method based on a low estimation of the
        bags covered by a concept under construction.
        As an experiment, percolation processes (forest fires typically) are simulated and the corresponding propagation models are learned based on a
        subset of observations.
        It reveals that the proposed MI approach is significantly more efficient on the task of propagation
        model recognition than existing methods.
    \end{abstract}

    \section{Introduction}\label{sec-intro}

Complex systems have been subject to more and more attentions for several years. They permit to model complex 
interactions between entities in a large specter of domains: organic ecosystems \citep{levin1998ecosystems, norberg2004biodiversity}, social networks \citep{levorato2014group}, air pollution \citep{athanasiadis2004agent}, 
%
%
%
%
energy management \citep{amin2005toward} or the world economy \citep{farmer2012economics, foster2005simplistic}. All these subjects are hard to model with standard tools but are critical for the human development.

Complex systems can be approached using different theories, such as fuzzy logic \citep{zadeh1973fuzzy}, graph theory 
\citep{pastor2015epidemic,hart2016graph}, rough set theory \citep{wu2004rough} and pretopology \citep{amor2006percolation} as major theoretical frameworks. The present work focuses on the pretopology theory, known for its ability to finely model
complex propagation phenomena. This theory is based on a propagation operator called \emph{pseudo-closure} which is the core of 
the pretopology theory. It is usually defined on a collection of neighborhoods and combines them in a subtle fashion.
This powerful operator offers a formalism able to describe, step by step, a non-trivial expansion process, such 
as a coalition formation in game theory \citep{auray1982fuzzy} or the spread of an epidemic \citep{liu2015complex} to mention but a few.

Most pretopology related works rely on a fixed hand pseudo-closure operator 
\citep{ahat2009pollution, van2013efficient,bui2014gesture}. In these articles, the pseudo-closure operator is defined as the conjunction of all neighborhoods.
The problem is that such an operator limits us to a single pretopological space which may not be 
suitable for each application case. Furthermore, in the common definition of the pseudo-closure operator, 
the same "credibility" and the same "usefulness" is given to each neighborhood. In applications where data is gathered from multiple
sources, this assumption is likely to be false.

Recent works address this problem by proposing more flexible solutions. Considering that a simple pretopological space cannot model any complex system, \citet{bui2015multi} introduce the concepts of
%
%
%
%
\emph{weak} and \emph{strong} pseudo-closure operators leading to \emph{thick} (i.e.~noisy) and \emph{thin} pretopological spaces respectively. In the same time, \citet{galindo2014pretopological} define a set of pseudo-closure operators, each referring to slightly different (topological) neighborhoods depending to a parameter $r$ such that the lower $r$, the thinnest the resulting pretopological space.
%
%
%
%

But these approaches to extend the pretopology theory remain limited. For instance,
there is no way to ignore or limit the impact of erroneous neighborhoods in favor of more reliable ones. Existing
works propose to use multiple pre-defined pretopological spaces, but it assure by no means that one of these spaces
or even a combination of these spaces \citep{dalud2009closed} suit the application needs.
Therefore, we need a methodology to build a pseudo-closure operator that fit with an expected behavior, given a collection of neighborhoods.
Because, in most cases, it is humanly infeasible to define such a complex pretopological 
space, the methodology has to be automatic and our proposal consists in learning the pseudo-closure operator.

\medskip

We introduce in this article a method to automatically learn a (V-type) pretopological space. 
Our method actually learns the underlying pseudo-closure operator, which defines the pretopological space.
\citet{cleuziou2015learning} were the first to tackle the problem of learning a pseudo-closure operator from 
observations. For a knowledge engineering purpose, they design a (semi)-supervised method to learn a pretopological space from which a lexical taxonomy is derived. To this end,
the authors first model the pseudo-closure operator as a linear combination of neighborhoods and then use a genetic algorithm that outputs a solution that matches best with a given (partial) structure.

The present paper intends to both generalize and enhance the original approach from \citet{cleuziou2015learning}. In this aim, three main contributions are introduced :
\begin{enumerate}
\item First of all, we generalize the problematic to any application case : the proposed generic method learns a propagation concept instead of focusing on building a structure (as a lexical taxonomy for example).
\item Then, the pseudo-closure operator is modeled in a logical formalism (rather than numerical) thus offering a wider space of combination solutions and, above all, facilitating the learning process which follows.
\item Lastly, but not leastly, the original stochastic learning strategy is depressed in favor of an efficient learning approach resulting in a multiple instance algorithm based on a greedy learning strategy. 
\end{enumerate}

The paper is organized as follows : the basics of pretopology, useful for the understanding of the problematic studied, are given in the next section (Section 2). The new modeling of the pseudo-closure operator as a logical combination (a positive disjunctive normal form (DNF)) of neighborhoods is then detailed in Section \ref{sec:pseudoclosure}.
The pseudo-closure operator is acquired by a supervised multiple instance learning 
process \citep{dietterich1997solving, zhou2004multi}. Multiple instance learning is particularly adapted
when an observation can be described by multiple instances. The need of this methodology (presented in Section \ref{sec:meth}) arises from
the fact that we learn a pretopological space based on its known \emph{elementary closed sets}. 
In the pretopology theory, a single closed set (considered as an observation in the learning problem) can be obtained by multiple propagations (multiple instances). In fact, the number of valid propagations/pseudo-closure operators is exponential in the size of the considered elementary closed sets. This leads us to a learning problem with an input set of examples exponential in size. We propose in Section \ref{sec:mi} a method to provide a (low) estimation of the number of covered positive bags (true positives) and the exact count of covered negative bags (false positives) without the need
to explicitly generate the full dataset.

Finally, Section \ref{sec:expe} is dedicated to quantitative and qualitative comparisons of our approach with the one presented in \citep{cleuziou2015learning} plus a greedy variation. We simulate a forest fire in a rectangular grid and we compare how each algorithm is capable of retrieving the underlying propagation model. The results confirm that, the new approach performs the best thanks to its more refined underlying quality measure, while keeping good performances in terms of execution time.

    \section{Basics of Pretopology}\label{sec-pretopo}

A lot of real world problems cannot be described with the tools provided by the standard topology theory.
So, the pretopology was born in the early 70s with the purpose of relaxing the topological axiomatic \citep{brissaud2011basics}.
The pretopology allows the study of the relations between elements of the powerset ${\cal P}(E)$.
A broad variety of works have shown that this theory is useful in many research fields : image analysis \citep{meziane1997satellite, bonnevay2009pretopological}
classification \citep{DBLP:conf/icpr/FrelicotL98, DBLP:journals/siu/AhatABJD10}, 
similarity or proximity detections tasks \citep{DBLP:conf/incos/LeTNP13},
or recognition of complex propagation (or expansion) phenomena \citep{DBLP:conf/rivf/AmorLL07, DBLP:conf/complex/AhatABLC09}. We put our work
in the latter area.

The pretopology theory deals with pretopological spaces which are defined as a couple $(E, a)$ with $E$ 
a non-empty finite set of elements and $a(.)$ a pseudo-closure operator. The pseudo-closure operator
$a(.)$ is a mapping from any set in ${\cal P}(E)$ to a larger one in ${\cal P}(E)$.

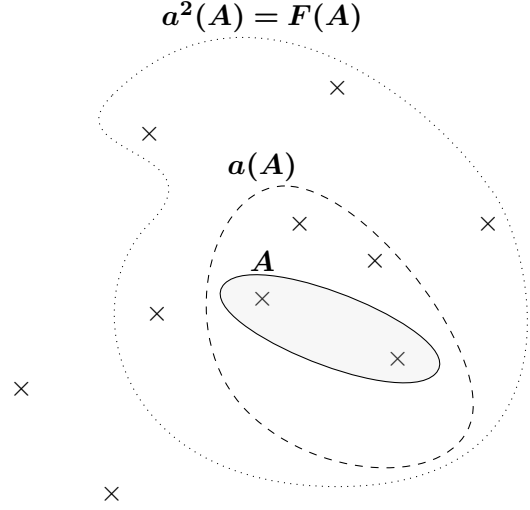
\begin{figure}
%
%
	\centering
  \begin{tikzpicture}
      \node (a) at (0, 0) {$\times$};
      \node (b) at (1.5, 0.5) {$\times$};
      \node (c) at (0.5, 1) {$\times$};
      \node (d) at (-3.2, -1.2) {$\times$};
      \node (e) at (-1.5, 2.2) {$\times$};
      \node (f) at (1.8, -0.8) {$\times$};
      \node (g) at (-1.4, -0.2) {$\times$};      
      \node (h) at (-2, -2.6) {$\times$};
      \node (i) at (3, 1) {$\times$};
      \node (j) at (1, 2.8) {$\times$};

      \coordinate (o1) at ($(a) + (-0.5, 0.2)$);
      \coordinate (o2) at ($(f) + (0.5, -0.2)$);

      \draw[fill=gray!30, fill opacity=0.2] (o1) 
      	..controls +(-0.5, -0.5) and +(-0.5, -0.5).. (o2)
      	..controls +(0.5, 0.5) and +(0.5, 0.5).. (o1);

      \coordinate (o3) at ($(c) + (-0.2, 0.5)$);
      \coordinate (o4) at ($(a) + (-0.2, -1.5)$);
      \coordinate (o5) at ($(o4) + (3, 0)$);

      \draw[dashed] (o3) 
      	..controls +(-1, 0) and +(-1, 1).. (o4)
      	..controls +(1, -1) and +(-0.1, -1).. (o5)
      	..controls +(0.1, 1) and +(1, 0).. (o3);

      \coordinate (o6) at ($(e) + (-0.5, 0.5)$);
      \coordinate (o7) at ($(a) + (+1, -2.5)$);
      \coordinate (o8) at (-1.5, 1);
      \coordinate (o9) at ($(i) + (0, 0.5)$);

	  \draw[dotted] (o6) 
      	..controls +(-0.8, -0.8) and +(1, 1).. (o8)
        ..controls +(-1, -1) and +(-3, 0).. (o7)
        ..controls +(3.5, 0) and +(0.5, -0.5).. (o9)
        ..controls +(-0.5, 0.8) and +(2, 2).. (o6);

      \node at (0, 0.5) {$\bm{A}$};
      \node at (0, 1.75) {$\bm{a(A)}$};
      \node at (0, 3.75) {$\bm{a^2(A) = F(A)}$};

  \end{tikzpicture}
  
  	\caption{The pseudo-closure expansion process, from a set $A$}
\end{figure}

\begin{definition}[Pseudo-closure operator]
	Let $(E, a)$ be a pretopological space. Then $a(.)$ is a function 
    $a : {\cal P}(E) \rightarrow {\cal P}(E)$ and is called a pseudo-closure operator.
    A pseudo-closure operator must respect two properties which are :
    \begin{enumerate}[label=\roman*)]
      \item $a(\emptyset) = \emptyset$
      \item $\forall A \in {\cal P}(E), A \subseteq a(A)$
	\end{enumerate}
\end{definition}


The pseudo-closure operator is not idempotent, it means that, unlike the closure operator in the topology theory, $a(A)$ is not necessarily equals to $a(a(A))$, where $A \in {\cal P}(E)$. The pretopological pseudo-closure operator is such that $\forall A \in {\cal P}(E), a(A) \subseteq a(a(A))$ (cf $ii$). Let $k$ be a positive integer and $A \in {\cal P}(E)$, we note $a^k(A)$ the $k$ successive 
applications of $a(.)$ on $A$ : $a^k = a \circ \overset{k-2}{\ldots} \circ a$.
Let $k$ a positive integer such as $a^k(A) = a^{k+1}(A)$, then $a^k(A)$ is called the closure of $A$, noted $F(A)$.

\begin{definition}[Closed set]
	Let $(E, a)$ be a pretopological space.
    Then $K \in {\cal P}(E)$ is a closed set if and only if $a(K) = K$.
    
    The closure of $A \in {\cal P}(E)$ is a closed set 
    noted $F(A)$ and is such that $\exists k, a^k(A) = F(A)$.
    If $A$ is a singleton, then $F(A)$ is called its \emph{elementary closure}.
\end{definition}

It is highly common to consider only a particular subset of pretopological spaces: 
V-type pretopological spaces. 
\begin{definition}[V-type pretopological space]
A pretopological space $(E, A)$ is V-typed if and only if its pseudo-closure operator $a(.)$ respects the 
isotonic property defined as follow :
\begin{enumerate}[label=\roman*),start=3]
	\setcounter{enumi}{2}
    \item $\forall A \in {\cal P}(E), \forall B \in {\cal P}(E), A \subseteq B \Rightarrow a(A) \subseteq a(B)$
\end{enumerate}
\end{definition}

These pretopological spaces have some good structuring properties, for example, 
the intersection of closed sets is a closed set ; \citet{DBLP:journals/isci/LargeronB02} rely on these properties to define an algorithm that structures the elements of a V-type pretopological space $(E, a)$ into a DAG. This algorithm is the cornerstone of the work of \citet{cleuziou2015learning} by considering a lexical taxonomy as a DAG structuring natural language terms.

The V-type isotonic property will be also of crucial importance for the pseudo-closure learning framework we propose in the remaining. This property allows us to efficiently learn a pretopological space: as we will see in section \ref{sec:mi}, it allows to accurately estimate the number of positive and negative bags covered by a solution under construction. As a consequence, we will consider only V-type pretopological spaces in the rest of this article.

    \section{Pseudo-closure operator}\label{sec:pseudoclosure}

The pseudo-closure operator is often defined by a neighborhood on a given set $E$.
Given a neighborhood function $V : E \rightarrow {\cal P}(E)$, it is possible to derive a pseudo-closure operator such as
$a(A) = \{ x \in E \mid V(x) \cap A \neq \emptyset \},~\forall A \in {\cal P}(E)$. This operator expands a set $A \in {\cal P}(E)$ to any element $x \in E$ whose neighborhood intersects $A$. Though it is a quite common way to define a pseudo-closure operator, it is not the most interesting way to capitalize on the pretopology theory.
Indeed, the pretopology theory reveals its power in the multiple-criteria context, that is to say when the 
pseudo-closure operator is defined by a combination of multiple neighborhoods.

So, considering a collection of $k$ neighborhoods functions ${\cal V} = \{V_1, \ldots, V_k\}$ on $E$, \citet{bui2015multi} define the concepts of \emph{strong} and \emph{weak} pseudo-closure operators combining the neighborhoods. These two pseudo-closure operators, noted respectively $a_{strong}$ and $a_{weak}$, are given by:
\begin{align*}
	a_{strong}(A) &= \{ x \in E \mid \forall V \in {\cal V}, V(x) \cap A \neq \emptyset \},~\forall A \in {\cal P}(E)\\
    a_{weak}(A) &= \{ x \in E \mid \exists V \in {\cal V}, V(x) \cap A \neq \emptyset \},~\forall A \in {\cal P}(E)
\end{align*}
Where $a_{strong}(A)$ expands $A$ to any element $x \in E$ such that all of its neighborhoods $V_i(x)$ intersect $A$ ; $a_{weak}(A)$ expanding $A$ to any element $x \in E$ such that at least one of the $k$ neighborhoods $V_i(x)$ intersects $A$.

In both cases, it is assumed that each neighborhood is equally necessary or sufficient for the expansion model. This assumption is not satisfied by the large majority of real world cases. Thus, research was carried out about a methodology aiming to build a pseudo-closure operator based on multiple neighborhoods and designed in such a way that useful neighborhoods are picked and combined together while erroneous ones are discarded.



\citet{cleuziou2015learning} were the first tackling this problem by proposing to learn the combination 
of neighborhoods that defines the pseudo-closure operator (and hence its associated pretopological space).
This section first recalls the (weighted) pseudo-closure definition proposed by \citet{cleuziou2015learning}, then points the limitation of such a numerical and linear modeling and finally details a new proposal based on the logical formalism.

\subsection{Weighted modeling of the pseudo-closure}

In \citep{cleuziou2015learning}, the authors learn a pretopological space, from which a lexical taxonomy is derived by a pretopological structuring algorithm \citep{DBLP:journals/isci/LargeronB02}.

Their approach consists in learning a vector of weights $\bm{w}=(w_0,\dots,w_k)$ where each weight $w_i$ is associated to one neighborhood $V_i \in {\cal V}$, and  $w_0$ is a bias parameter. They define the notion of \emph{parametrized pseudo-closure} as a pseudo-closure
operator based on the vector $\bm{w}$, noted $a_{\bm{w}}(.)$ and defined by
\begin{equation}
    a_{\bm{w}}(A) = \{ x \in E \mid \sum_{i=1}^{k} w_i \cdot \mathbbm{1}_{V_i(x) \cap A \neq \emptyset} \ge w_0 \}, \forall A \in {\cal P}(E)
\end{equation}

Such a parametrized pseudo-closure expands a subset $A \in {\cal P}(E)$ to the elements $x \in E$ for which there exists a subset of neighborhoods $\cal V' \subseteq V$ such that :
\begin{enumerate}
\item any neighborhood in ${\cal V}'$ intersects $A$ ($V(x) \cap A \neq \emptyset, \forall V \in {\cal V}'$) and 
\item the sum of the weights associated to the neighborhoods in $\cal V'$ is greater than $w_0$.
\end{enumerate}

\citet{cleuziou2015learning} show that the pretopological space $(E, a_{\bm{w}})$ is guaranteed to be V-typed if and only if $\bm{w}$ satisfies the three following constraints :
\[w_0 > 0 ~~,~~\sum_{i=1}^{k}w_i \ge w_0 ~~\text{and}~~w_i \ge 0, \forall i\]

\medskip


This definition of the pseudo-closure operator generalizes existing combination-based operators since it allows to retrieve both the {\em weak} and the {\em strong} pseudo-closure operators from \citet{bui2015multi}.
On the one hand, the weak pseudo-closure operator is obtained by giving a value greater than $w_0$ to each weight ; as a consequence, each neighborhood in $\cal V$ is considered as a sufficient (trustworthy) justification to trigger the expansion process.
On the other hand, the strong pseudo-closure operator is obtained when every weight is required to stand the threshold $w_0$ (for example, giving the value $\frac{w_0}{k}$ to each weight).
Thus, the propagation is performed if and only if each neighborhood $V \in {\cal V}$ allows it.

The expressiveness is not reduced by switching from a standard pseudo-closure operator to a weighted modeling. Even better, 
the new modeling actually brings (a lot) in expressiveness because many more pretopological spaces can be represented by
attributing any other valid value to the vector $\bm{w}$.
Furthermore, this representation allows to give more or less importance to a neighborhood $V_i$ according to its weight $w_i$: a small (resp. great) value modeling a low (resp. high) confidence or significance assigned to the corresponding neighborhood.

\bigskip

Although learning a parameterized pseudo-closure operator is a great evolution in the pretopology theory uses, two major limitations remain. First, the considered hypothesis space is unadapted to the problem :
there is a finite number of pretopological spaces defined on $E$, but the considered set of weighting vectors is infinite. It implies a lot of redundancy in the hypothesis space and it makes its exploration inefficient. Then, the expressiveness of the linear modeling is limited. The parameterized pseudo-closure operator defines the propagation concept
(i.e. $x \in a_{\bm{w}}(A)$), by the following threshold function $f()$:
\begin{equation*}
	\forall A \in {\cal P}(E), x \in E,
	f(A, x) = 
    \begin{cases}
    	1 \text{ if } \sum_{\{ i \mid V_i(x) \cap A \neq \emptyset \}} w_i \ge w_0 \\
        0 \text{ otherwise}
    \end{cases}
\end{equation*}
Then $a_{\bm{w}}(.)$ can be expressed as simply as $\forall A \in {\cal P}(E), a_{\bm{w}}(A) = \{ x \in E \mid f(A, x) = 1\}$.
Such a linear function defines an hyperplane that makes only linearly separable problems to be solved. Unfortunately, not every propagation phenomena can be expressed as a linear model.
%
%
%
%


\begin{figure}
    \centering

    \begin{tikzpicture}[graph_name/.style={}, graph_node/.style={circle, draw, outer sep=1mm}, >=stealth]
        \coordinate (xshift) at (4, 0);
        \coordinate (yshift) at (0, -4);

        \node (ref1) at (0, 2) {$S^*_1$};
        \node (refa1) at (1, 1.85) {a};
        \node (refb1) at (0.5, 1) {b};
        \node (refc1) at (1.5, 1) {c};
        
        
        \node[circle, draw, inner sep=0.5mm, outer sep=1mm] (refd1) at (1, 0.5) {d};
		\node[ellipse, inner sep=0mm, draw, rotate fit=45, fit={(refc1) (refd1)}] (Fc1) {};
        \node[ellipse, inner sep=0.2mm, draw, rotate fit=20, fit={(Fc1) (refb1)}] (Fb1) {};
		\node[ellipse, inner sep=0mm, draw, rotate fit=40, inner sep=-1mm, fit={(Fb1) (refa1)}] {};

        \node (ref2) at ($(ref1) + (yshift)$) {$S^*_2$};
        \node (refa2) at ($(refa1) + (yshift)$) {a};
        \node (refb2) at ($(refb1) + (yshift)$) {b};
        \node[circle, draw, inner sep=0.5mm, outer sep=1mm] (refc2) at ($(refc1) + (yshift)$) {c};
        \node[circle, draw, inner sep=0.5mm, outer sep=1mm] (refd2) at ($(refd1) + (yshift)$) {d};

        \node[ellipse, inner sep=0.2mm, draw, rotate fit=20, fit={(refb2) (refc2) (refd2)}] (Fb2) {};
		\node[ellipse, inner sep=0mm, draw, rotate fit=40, inner sep=-1mm, fit={(Fb2) (refa2)}] {};

        \node[graph_name] (V1) at ($(ref1) + (xshift)$) {$V_1$};
        \node[graph_node] (a1) at ($(1, 2) + (xshift)$) {a};
        \node[graph_node] (b1) at ($(0, 1) + (xshift)$) {b};
        \node[graph_node] (c1) at ($(2, 1) + (xshift)$) {c};
        \node[graph_node] (d1) at ($(1, 0) + (xshift)$) {d};
        \path[->] 
            (a1) edge[bend right=10] (b1) 
            (b1) edge[bend right=10] (a1) 
            (d1) edge[bend left=10] (b1);

        \node[graph_name] (V2) at ($(V1) + (xshift)$) {$V_2$};
        \node[graph_node] (a2) at ($(a1) + (xshift)$) {a};
        \node[graph_node] (b2) at ($(b1) + (xshift)$) {b};
        \node[graph_node] (c2) at ($(c1) + (xshift)$) {c};
        \node[graph_node] (d2) at ($(d1) + (xshift)$) {d};
        \draw[->] 
            (b2) edge[bend right=10] (a2) 
            (b2) edge[bend right=10] (c2)
            (d2) edge[bend right=10] (c2);


        \node[graph_name] (V3) at ($(V1) + (yshift)$) {$V_3$};
        \node[graph_node] (a3) at ($(a1) + (yshift)$) {a};
        \node[graph_node] (b3) at ($(b1) + (yshift)$) {b};
        \node[graph_node] (c3) at ($(c1) + (yshift)$) {c};
        \node[graph_node] (d3) at ($(d1) + (yshift)$) {d};
        \path[->] 
            (d3) edge[bend right=10] (c3)
            (c3) edge[bend right=10] (b3);

        \node[graph_name] (V4) at ($(V1) + (xshift) + (yshift)$) {$V_4$};
        \node[graph_node] (a4) at ($(a1) + (xshift) + (yshift)$) {a};
        \node[graph_node] (b4) at ($(b1) + (xshift) + (yshift)$) {b};
        \node[graph_node] (c4) at ($(c1) + (xshift) + (yshift)$) {c};
        \node[graph_node] (d4) at ($(d1) + (xshift) + (yshift)$) {d};
        \path[->] 
            (a4) edge[bend right=10] (b4) 
            (b4) edge[bend right=10] (c4)
            (c4) edge[bend right=10] (b4);

        \draw (a1.north) ++ (2, 0) -- ($(d3.south) + (2, 0)$);
        \draw (b1.west) ++ (0, -2) -- ($(c2.east) - (0, 2)$);
    \end{tikzpicture}

    \caption{Two target pretopological spaces $(S^*_1,S^*_2)$ and four neighborhood relations $(V_1,V_2,V_3,V_4)$ on a set $E=\{a,b,c,d\}$.}
    \label{fig:example}
\end{figure}
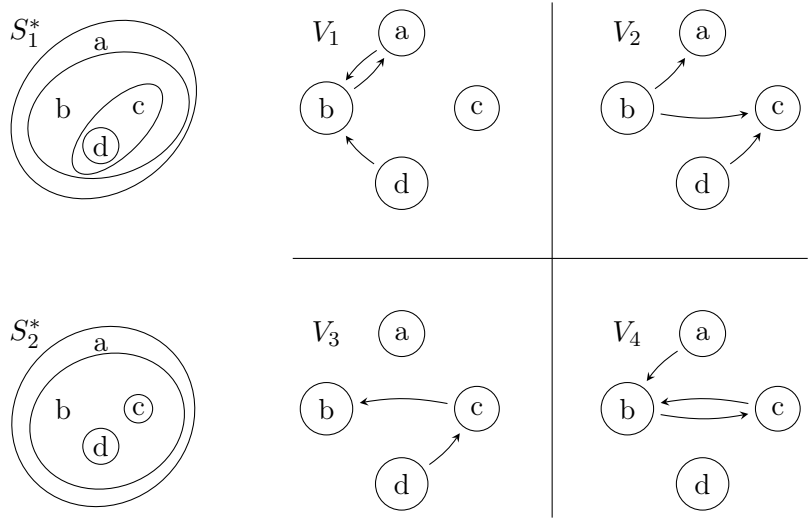

As an example, let's consider the collection of neighborhoods ${\cal V} = \{V_1, V_2, V_3, V_4\}$ and the two 
V-type compliant sets of elementary closures $S^*_1$ and $S^*_2$ illustrated in Figure \ref{fig:example}.
The four neighborhoods are represented as four directed graphs such that an edge from $x$ to $y$ in $V_i$ means $y\in V_i(x)$ ; reflexive edges being implicit\footnote{In order to ensure the required growth property ($x \in a(x),~\forall x \in E$), any node is considered having a reflexive edge, 
which is not represented on the figure for the sake of legibility.}. We consider the problem of defining one neighborhood combination for each of the two target pretopological structuring $S^*_1$ and $S^*_2$, whose elementary closures are explained in Table \ref{tab:closures}.

\begin{table}
  \centering
  \begin{tabular}{c|c|c}
      $x \in E$ & $F^*_1(x)$ & $F^*_2(x)$ \\
      \hline

      $a$ & $\{ a, b, c, d \}$ & $\{a, b, c, d \}$\\    
      $b$ & $\{ b, c, d \}$ & $\{ b, c, d \}$\\
      $c$ & $\{ c, d \}$ & $\{ c \}$\\
      $d$ & $\{ d \}$ & $\{ d \}$\\
  \end{tabular}
  
  \caption{Elementary closures of $S^*_1$ and $S^*_2$}
  \label{tab:closures}
\end{table}


In order to illustrate the two limitations stated above (redundancy and expressiveness),
we show that the linear weighting approach exposes an infinity of equivalent solutions to retrieve $S^*_1$ (redundancy) but no solution to build $S^*_2$ while non-linear combinations are suitable (lack of expressiveness).

\noindent{\bf Redundancy.} A solution that retrieves perfectly the target structuring $S^*_1$ is the 
vector $\bm{w} = (1, 0.5, 0.5, 1, 0)$. It is easier (and equivalent) to represent $\bm{w}$ by the 
logical formula $V_1 \wedge V_2 \vee V_3$ meaning that $x \in a_{\bm{w}}(A)$ if and only if at least one of the two following conditions is satisfied :
\begin{enumerate}
\item both $V_1$ and $V_2$ allows it (i.e. both $V_1(x)$ and $V_2(x)$ intersect $A$, not necessarily on the same element(s)),
%
%
%
%
\item $V_3$ allows it (i.e. $V_3(x)$ intersects $A$).
\end{enumerate}

The vector $\bm{w}$ is then used to define the pseudo-closure operator $a_{\bm{w}}(.)$ 
behaving as follow:
\begin{align*}
	\{a\} &\overset{V_1 \wedge V_2}{\rightarrow} \{a, b\} 
    	\overset{V_3}{\rightarrow} \{a, b, c \}
        \overset{V_1 \wedge V_2 \text{ and } V_3}{\rightarrow} \{a, b, c, d \}=F^*_1(\{a\})\\
    \{b\} &\overset{V_3}{\rightarrow} \{b, c \}
    	\overset{V_1 \wedge V_2 \text{ and } V_3}{\rightarrow} \{b, c, d\}=F^*_1(\{b\}) \\
    \{c\} &\overset{V_3}{\rightarrow} \{c, d\}=F^*_1(\{c\})\\
    \{d\} &\overset{\emptyset}{\rightarrow} \{d\}=F^*_1(\{d\})
\end{align*}
where an arrow describes one propagation step from an element $x \in E$ to its closure $F_{\bm{w}}(\{x\})$, and the clauses above each are the ones causing the propagation. Let us observe that it is easy to find another vector $\bm{w'}\neq \bm{w}$ defining an operator $a_{\bm{w'}}$ similar to $a_{\bm{w}}$ (e.g. $\bm{w'} = (1, 0.4, 0.6, 1, 0)$).
%
%
%
%
Actually, the operator $a_{\bm{w}}$ is derived from the (infinite) set of vectors on $[0,1]^5$ satisfying the following constraints : $$w_0 > 0,~~ w_1 < w_0,~~ w_2 < w_0,~~w_1 + w_2 \ge w_0,~~ w_3 \ge w_0,~~ w_4 < w_0-\max\{w_1,w_2\}$$


\noindent{\bf Expressiveness.} The target structuring $S^*_2$ is close to $S^*_1$ since it differs only on the closure $F^*_2(\{c\})$ that excludes $d$. By observing that $V_3$ were responsible of the propagation $\{c\} \rightarrow \{c,d\}$ in the previous model, we build a new solution by using $V_4$ to constrain the second propagation condition thus leading to the new satisfying logical formula $(V_1 \wedge V_2) \vee (V_3 \wedge V_4)$. One can show that this formula is the only satisfying solution and it is well known that the corresponding boolean function $v_1v_2+v_3v_4$ is not a linearly separable function and thus cannot be formalized with the weighting modeling.

\subsection{A logical approach}

To overcome the two main limitations of the weighted modeling exposed previously (redundancy and expressiveness), we propose to define a combination of neighborhoods as a logical formula in a disjunctive normal form (DNF). We limit our work to positive DNFs, that is to say DNFs without negation of literals.

Such a modeling requires to translate the information provided by a neighborhood to a logical concept ; so we define one logical predicate $q_i : {\cal P}(E) \times E \rightarrow \{0, 1\}$ for each neighborhood $V_i$, determining whether the neighborhood $V_i(x)$ of an element $x$ intersects a subset $A\subseteq E$.

\begin{equation}
	\forall V_i \in {\cal V}, A \in {\cal P}(E), x \in E,\, q_i(A, x) \equiv V_i(x) \cap A \neq \emptyset
\end{equation}
%


%
Then, given a positive DNF $Q$ defined on the language of predicates $\{q_i\}_{i=1}^k$, 
%
we define the pseudo-closure operator $a_Q(.)$ derived from $Q$ :
\begin{equation}
	\forall A \in {\cal P}(E), a_Q(A) = \{ x \in E \mid Q(A, x) \}
\end{equation}

\begin{theorem}
Let $Q$ be a positive DNF, the pretopological space $(E, a_Q)$ derived from $Q$ is of V-type.
\end{theorem}

\begin{proof}
	The pretopological space $(E, a_Q)$ is of V-type 
    if and only if
    $\forall A,B \in {\cal P}(E), A \subseteq B \Rightarrow a_Q(A) \subseteq a_Q(B)$ (isotony property).
    Let $c_j$ be the $j$-th (conjunctive) clause of $Q$ and $q_{ji}$ the $i$-th literal in $c_j$.
    \begin{align*}
    	x \in a_Q(A) &\Leftrightarrow Q(A, x) \\
        &\Leftrightarrow \exists j, c_j(A, x) \\
        &\Leftrightarrow \exists j \forall i, q_{ji}(A, x)
    \end{align*} 
    Noticing that the predicates $q_{ji}(A, x)$ are increasing functions (if $A \subseteq B$ then $q_{ji}(A, x) \Rightarrow q_{ji}(B, x)$) this property is satisfied by any conjunction $c_j(A,x)$ and then by the DNF $Q(A,x)$. 
    So $(E, a_Q)$ is a V-type pretopological space.
\end{proof}

\medskip

Let us notice that, in the following, we only consider well-formed (or simplified) DNFs. That is to say, DNFs satisfying to the 
following properties:
\begin{itemize}
	\item A clause is present only once
    \item A clause is not subsumed by another (more general) one
\end{itemize}
For instance, $Q_1 = q_1 \vee q_1$ and $Q_2 = q_1 \wedge q_2 \vee q_1$ are both ill-formed DNF.
$Q_1$ is ill-formed since the clause $q_1$ is present two times and
$Q_2$ is ill-formed too because $q_1 \wedge q_2 \Rightarrow q_1$, thus $q_1 \wedge q_2$ is subsumed by $q_1$.
$Q_3 = q_1$ is the only well-formed DNF equivalent to both $Q_1$ and $Q_2$.
This simplification reduces the hypothesis space since it would, as the weight vector model, be 
infinite otherwise.

\medskip

\newpage

We have just introduced a new logical model to represent a combination of neighborhoods $\cal V$ on a 
given finite non-empty set $E$.
This combination is a logical formula in positive disjunctive normal form used to define
a decision function $Q(A, x)$.
This decision function allows to build a V-type pretopological space $(E, a_Q)$ by defining the result 
of the pseudo-closure operator applied on a subset $A$ as the set of all elements ($x$) satisfying the decision function $Q(A,x)$.
Let us now show that the new proposed model is more expressive than the previous weighted model.

\begin{property}
	Any weighted pseudo-closure operator $a_w()$ can be reformulated as a logical pseudo-closure $a_Q()$.
\end{property}

\begin{proof}
	Let's consider a set $E$, a collection of $k$ neighborhoods $\cal V$, a weight vector $\bm{w} = (w_0, w_1, \ldots, w_k)$ and a positive DNF $Q$ built by putting one (conjunctive) clause for each combination of weights whose sum is greater or equal to $w_0$. Then the threshold function $f$ is equivalent to $Q$.
    
When $f(A, x) = 1$, it means that there is a combination of neighborhoods ${\cal V}' = \{ V \in {\cal V} \mid V(x) \cap A \neq \emptyset \}$ such that the sum of their associated weights is greater than $w_0$. By how we built $Q$, the logical conjunction $\bigwedge_{V_i \in {\cal V}'}q_i(A, x)$ belongs to $Q$ and is true (by definition of $\cal V'$).

On the contrary, $f(A, x) = 0$ means that there is no combination of neighborhoods ${\cal V}' = \{ V \in {\cal V} \mid V(x) \cap A \neq \emptyset \}$ such that the sum of their associated weights is greater than $w_0$. By how we built $Q$, there is also no conjunction $c$ such that $c(A, x) = 1$.
\end{proof}

As an example, let's consider the vector $\bm{w_1} = (1, 0.5, 0.5, 1)$ defined previously
to retrieve the elementary closed sets of $S^*_1$ (we omit $w_4=0$ for simplicity purposes). 
There are five combinations of weights whose sum is greater or equal to $w_0$ : $(w_1, w_2)$, $(w_1, w_3)$, $(w_2, w_3)$, $(w_3)$ and $(w_1, w_2, w_3)$. So we define our equivalent DNF to be 
$Q_1 = (q_1 \wedge q_2) \vee (q_1 \wedge q_3) \vee (q_2 \wedge q_3) \vee (q_3) \vee (q_1 \wedge q_2 \wedge q_3)$.
This DNF is obviously ill-formed, but it has the same meaning than $\bm{w_1}$. We simplify
this formula to be well-formed (and more readable): $Q_1 = (q_1 \wedge q_2) \vee q_3$.


Property 1 establishes that the logical pseudo-closure modeling is at least as expressive than the weighted model. But one of the strong motivation of such a new formalization is its ability to improve the expressiveness since some DNF cannot be expressed as a weight vector.
We encountered such a case earlier when we tried to find a linear model capable of retrieving the elementary closed sets in $S^*_2$. We already proposed a DNF suiting our 
needs: $Q_2 = (q_1 \wedge q_2) \vee (q_3 \wedge q_4)$ that cannot be modeled as a weight vector.

To conclude on this section, we introduced a simple and versatile model for the pseudo-closure operator that is based on a logical
formalism. Our model is versatile because it is capable to express a wider variety 
of V-type pretopological spaces than existing modeling. It is simple since the space of all the DNF is finite and easier to explore contrary to the (infinite) space of weighted vectors. It is also simple because a DNF is much more human understandable than a weight vector.

    \section{Learning methodologies}\label{sec:meth}

\citet{cleuziou2015learning} propose the \emph{Learn Pretopological Space} (LPS) algorithm which makes use of a genetic algorithm. Their method relies on a quality measure, we call it \emph{extrinsic measure}, which is maximized by the outputted solution. We expose in this section our approach to learn a positive DNF by means of the same genetic algorithm, but in a supervised context, instead of a semi-supervised context as in the original work.
Because this kind of algorithm suffers from high complexity, we then introduce a greedy learning approach that reduces the complexity issue. Finally, we show that the extrinsic measure is not suitable for such a greedy approach in the context of the LPS task. 

\subsection{The extrinsic measure}

Given a finite set of elements $E$, a set of target elementary closed sets $S^*=\{F^*(\{ x \})\}_{x\in E}$ and a collection of neighborhoods $\cal V$, the LPS task consists in learning a pretopological space $(E,a)$ whose underlying structuring $S=\{F(\{ x \})\}_{x \in E}$ (i.e.~its elementary closed sets) best matches with $S^*$. A quality measure, quantifying the degree of matching, is thus
necessary for leading the learning process.

\citet{cleuziou2015learning} resorted to the F-measure that computes a trade off  between the accuracy and the completeness in the comparison of two sets of closures.
The accuracy is defined as the precision of the solution, and the completeness by the recall of the solution:
\begin{align*}
    Precision(S^*, S) &= \frac{\sum_{x \in E}|F^*(\{ x \}) \cap F(\{ x \})|}{\sum_{x \in E}|F(\{ x \})|} \\
    Recall(S^*, S) &= \frac{\sum_{x \in E}|F^*(\{ x \}) \cap F(\{ x \})|}{\sum_{x \in E}|F^*(\{ x \})|}
\end{align*}
The precision (resp. recall) is defined as the ratio between the number of correctly retrieved
elements in the learned closures ($S$) and the total number of elements in the learned closures $S$ (resp. in the target closures $S^*$).
The final F-measure is defined as the harmonic mean between the precision and the recall: $2 \cdot \frac{Recall \cdot Precision}{Recall + Precision}$. Figure \ref{fig:measures} illustrates the computation of the F-measure by comparing a "closure function"\footnote{The term {\em closure function} refers to the function that associates to each element $x$ its (elementary) closure $F(x)$.} $F(\{ x \})$ to a target one $F^*(\{ x \})$ on a toy set $E$. 

\begin{figure}
    \begin{minipage}{.54\textwidth}
    	\centering
        \begin{tabular}{c|c|c|c}
            $x \in E$ & $F^*(\{ x \})$ & $F(\{ x \})$ & $\left| F^*(\{ x \}) \cap F(\{ x \}) \right|$\\
            \hline
            $a$ & $\{a, b, c, d\}$ & $\{a, c, d\}$ & 3\\
            $b$ & $\{b, c, d\}$ & $\{a, b, c, d\}$ & 3\\
            $c$ & $\{c\}$ & $\{c, d\}$ & 1\\
            $d$ & $\{d\}$ & $\{c, d\}$ & 1\\
        \end{tabular}
    \end{minipage}
    \begin{minipage}{.45\textwidth}
    	\centering
        \begin{align*}
            Precision &= \frac{8}{11} \approx 0.72 \\
            Recall &= \frac{8}{9} \approx 0.88 \\
            F\!-\!measure &\approx 0.79
        \end{align*}
    \end{minipage}
    \caption{Example of F-measure computation between two sets of closures.}
    \label{fig:measures}
\end{figure}

\medskip


In the specific context of the LPS problematic, we name such a quality measure as \emph{extrinsic measures} since it evaluates only (with an external point of view) the final elementary closures revealed by the structuring processes, rather than the (internal or \emph{intrinsic}) structuring process itself.

\subsection{Genetic LPS}

For clarity purposes, we present two variants of LPS : the Numerical Genetic LPS (the original one from \citet{cleuziou2015learning}) and
the Logical Genetic LPS whose output is a DNF-based pretopological space. 
We will refer to both of these algorithms as the Genetic LPS framework.

The Genetic LPS framework aims to learn a pretopological space by mean of genetic algorithms.
Given a set of target elementary closed sets $S^*$ and a collection of neighborhoods $\cal V$, it outputs the combination of 
neighborhoods (a vector or a DNF, according to which algorithm is considered) defining a pseudo-closure operator leading to the highest extrinsic measure on the resulting elementary closed sets $S$. Algorithm \ref{algo:glps} provides a pseudo-code of the generic Genetic LPS framework.



In order to compute the extrinsic quality of a solution $Q$, the set $S_Q$ of the elementary closed sets in the pretopological space $(E, a_Q)$ must be computed. This structuring step is the most expensive operation we rely on, hence we choose to define the complexity of the LPS algorithms in term of the number of required structuring steps. 

Genetic algorithms have a tendency to need a lot of iterations to output their solution, making their execution a bit slow. The quality of their outputs is often tied to the size of their initial population, which is itself related to the time needed to converge toward an acceptable solution. For example, Algorithm \ref{algo:glps} exposes a complexity (in terms of number of structuring) of $O(initial\_pop \cdot max\_iter)$ with $initial\_pop$ usually high.
Moreover, due to the stochastic nature of these algorithms, multiple executions can lead to different outputs.
These are the reasons why we introduce in the next subsection a less expensive (greedy) learning approach.

\begin{algorithm}
	\DontPrintSemicolon
    \SetKwProg{Fn}{Function}{}{end}
    
    \Fn{GeneticLPS (E, $S^*$, $\cal V$, max\_iter, initial\_pop, required\_iter\_convergence)}{
    	$iter \gets 0$ \;
    	$iter\_conv \gets 0$ \;
        $score \gets 0$ \;
        $is\_terminated \gets false$ \;
        $\tilde{Q} \gets \emptyset$ \;
        
        $population \gets initial\_pop$ random solutions \;
        
        \While{$iter < max\_iter$ and $\neg is\_terminated$}{
			$iter \gets iter + 1$ \;        
        	$scores \gets []$ \;
            \ForEach{$Q \in population$}{
            	$S_Q \gets structuring(E,{\cal V},Q)$\; 
            	Append $extrinsic\_measure(S^*,S_Q)$ to $scores$ \;
            }
            
            $best\_score \gets$ highest value in $scores$ \;
            $best\_individual \gets$ the individual exposing the highest score \;
            \If{$best\_score = score$}{
				$iter\_conv \gets iter\_conv + 1$ \;
                \If{$iter\_conv = required\_iter\_convergence$}{
                	$is\_terminated \gets true$ \;
                }
            }
            \ElseIf{$best\_score > score$}{
				$iter\_conv \gets 1$ \;
            	$score \gets best\_score$ \;
                $\tilde{Q} \gets best\_individual$ \;
            }
            
            $population \gets$ cross and mutate best individuals \;
        }
        
        \Return{$\tilde{Q}$}
    }
    
    \caption{Genetic LPS}
    \label{algo:glps}
\end{algorithm}

\subsection{Greedy LPS}

Greedy LPS is a greedy variant of the Logical Genetic LPS algorithm: given a set of target elementary closed sets $S^*$ and a collection of neighborhoods $\cal V$, it outputs a combination of the neighborhoods in $\cal V$ as a positive DNF $Q$. It uses a greedy heuristic that aims to accelerate the learning process.

\begin{algorithm}
	\DontPrintSemicolon
    \SetKwProg{Fn}{Function}{}{end}
    
    \Fn{best\_clause (E, $S^*$, $\cal V$, Preds, Q, base\_clause)}{
    	\If{$|base\_clause| \ge |Preds|$}{
        	\Return{$\emptyset$} \;
        }
    	
        $best\_clause \gets \emptyset$ \;
        $best\_score \gets 0$ \;

        \ForEach{$q \in Preds \setminus base\_clause$}{
        	$candidate\_clause \gets Q \vee (base\_clause \wedge q)$ \;
            $S_Q \gets structuring(E, {\cal V}, candidate\_clause)$ \;
            $score \gets extrinsic\_measure(S^*, S_Q)$ \;
            
            \If{$score > best\_score$}{
            	$best\_score \gets score$ \;
                $best\_clause \gets candidate\_clause$ \;
            }
        }
        
        $best\_specialized\_clause \gets best\_clause(E,S^*,{\cal V}, Preds, Q, best\_clause)$ \;
        
        \If{$best\_specialized\_clause \neq \emptyset$}{
        	$S_Q \gets structuring(E, {\cal V}, best\_specialized\_clause)$ \;
            $specialized\_clause\_score \gets extrinsic\_measure(S^*, S_Q)$ \;
            
            \If{$best\_score < specialized\_clause\_score$}{
            	$best\_clause \gets best\_specialized\_clause$ \;
            }
        }
        
    	\Return{$best\_clause$} \;
    }
    
    \caption{Building the "best" clause}
    \label{algo:findclause}
\end{algorithm}

\begin{algorithm}
	\DontPrintSemicolon
    \SetKwProg{Fn}{Function}{}{end}
    
    \Fn{GreedyLPS(E, $S^*$, $\cal V$, max\_iter)}{
    	$iter \gets 0$ \;
        $score \gets 0$ \;
        $is\_terminated \gets false$ \;
        $Preds \gets $ predicates derived from $\cal V$ \;
        $\tilde{Q} \gets \emptyset$ \;
        
        \While{$iter < max\_iter$ and $\neg is\_terminated$}{
        	$clause \gets best\_clause(E, S^*, {\cal V}, Preds, \tilde{Q}, \emptyset)$ \;
            $S_Q \gets structuring(E, {\cal V}, \tilde{Q} \vee clause)$ \; 

            \If{($clause = \emptyset$ or $extrinsic\_measure(S^*, S_Q) \le score$}{
            	$is\_terminated \gets true$ \;
            }
            \Else{
                $score \gets extrinsic\_measure(S^*, S_Q)$ \;
            	$\tilde{Q} \gets \tilde{Q} \vee clause$ \;
			}
        }
        
        \Return{$\tilde{Q}$}
    }
    
    \caption{Greedy LPS}
    \label{algo:greedy}
\end{algorithm}

The algorithm starts with an empty DNF $Q$ which is built in an iterative fashion.
At each iteration, the algorithm performs a beam search to find the clause $c$
maximizing the extrinsic measure defined above. Then $c$ is appended to $Q$ and the next
iteration begins.
The algorithm terminates either when the maximum number of iterations is reached or when $Q$ cannot be optimized, that is to say no additional clause improves the extrinsic measure obtained by $a_Q(.)$.
The pseudo-code of a simplified version (with a beam search of size 1) of Greedy LPS is presented in Algorithm \ref{algo:greedy}. It exposes a complexity of $O(max\_iter \cdot |{\cal V}| \cdot beam\_size)$ which is in practice much less than the one exposed by Genetic LPS since $|{\cal V}| \cdot beam\_size$ is usually substantially below the size of the initial population of Genetic LPS.

Greedy LPS is much simpler than the Genetic LPS framework
and it results in much shorter execution times. In addition, further experimental comparisons (cf. section \ref{sec:expe}) are going to reveal that - despite the lack 
of completeness in the exploration of the solution space and thanks to the concision of the outputted pseudo-closures - the greedy learning strategy outperforms significantly the stochastic approaches in the LPS task. 

\medskip

The postulate in the following contribution is that the extrinsic measure used in the previous learning methodologies is not adapted for incremental learning strategies. The greedy LPS approach should retrieve even better models by using an objective measure considering not only the quality of the structuring (resulting elementary closed sets) but also the potential of the pseudo-closure (elementary and non-elementary closed sets).

As an example, suppose we try to learn the pretopological space showed in figure \ref{fig:closurecomparison}, and suppose two closure operators $F_Q(.)$ and $F_{Q'}(.)$ defined from two DNF $Q$ and $Q'$.
Both $Q$ an $Q'$ have the same (extrinsic) F-measure score, as a consequence the Greedy LPS method is unable to determine which DNF is the best. Choosing the wrong clause in the iteration $i$ could have a huge impact in the clause chosen in iteration $i+1$ (and further). Therefore it is fundamental to be able to determine which of $Q$ or $Q'$ is the most useful.

Recall that we only consider V-type pretopological spaces, so $\forall A, \forall B, A \subseteq B \Rightarrow a(A) \subseteq a(B)$. Thus, the fact that $F_Q(\{a\})=\{a, d\}$ enlighten us on the closures of the supersets of $\{a\}$ : any set $A \in {\cal P}(E)$ such that $a \in A$ will propagate to at least $d$. Of course, the same apply with $Q'$: $\{c\}$ and its supersets are expanded to $d$. Although the informations provided by $Q$ and $Q'$ are quantitatively equivalent, $Q$ gives more useful informations. Indeed, $Q'$ tells us that, because $d \in F_{Q'}(\{c\})$, then $d \in F_{Q'}(\{b, c\})$. But, since we try to learn a model able to retrieve the elementary closed sets provided as input, this information is useless: we know that $F^*(\{c\}) = \{c, d\}$, then we must learn a formula such that $\{c\}$ propagates to $d$ (and only $d$), otherwise the learned elementary closed set would be wrong. So knowing how $\{b, c\}$ will propagate is of no interest for retrieving the target elementary closed sets. On the contrary, $Q$ tells us that each superset of $\{a\}$ expands to $d$, since $F^*(\{a\})=E$, all of these informations are precious to learn the target pretopological space. Thus $Q$ gives more qualitative informations than $Q'$.

We need a different quality measure taking into account these \emph{intrinsics informations} because such a method would be able to determine that, while $Q$ an $Q'$ are equivalent in terms of F-measure (aka the extrinsic quality measure), $Q$ has more potential with regards to the subsequent iterations. We talk about "potential" because although $Q$ informs us that $d \in F_Q(\{a, c\})$, this information is not reflected by the resulting elementary closed sets. However, it can, and will, be important when learning the next clauses since it will influence the algorithm to learn a clause that completes the elementary closure of $\{a\}$. On the contrary, the extrinsic quality measure cannot influence the results of the following iterations in such a way because it considers any expansion to be as important, which is wrong as we just show.

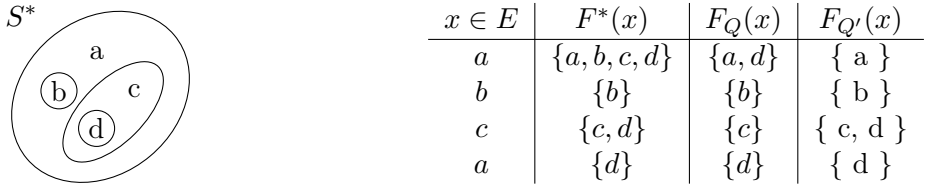
\begin{figure}
    \begin{minipage}{0.33\textwidth}
        \centering
        \begin{tikzpicture}[graph_name/.style={}, graph_node/.style={circle, draw, outer sep=1mm}, >=stealth]
            \coordinate (xshift) at (4, 0);
            \coordinate (yshift) at (0, -4);

            \node (ref1) at (0, 2) {$S^*$};
            \node (refa1) at (1, 1.5) {a};
            \node[circle, draw, inner sep=0.5mm, outer sep=1mm] (refb1) at (0.5, 1) {b};
            \node (refc1) at (1.5, 1) {c};
            \node[circle, draw, inner sep=0.5mm, outer sep=1mm] (refd1) at (1, 0.5) {d};
            
            \node[ellipse, inner sep=0.5mm, draw, rotate fit=45, fit={(refc1) (refd1)}] (Fc1) {};
            \node[ellipse, inner sep=0.5mm, draw, rotate fit=40, fit={(refa1) (refb1) (Fc1)}] {};
        \end{tikzpicture}
    \end{minipage}
    \begin{minipage}{0.66\textwidth}
        \centering
        \begin{tabular}{c|c|c|c}
            $x \in E$ & $F^*(x)$ & $F_Q(x)$ & $F_{Q'}(x)$ \\
            \hline

            $a$ & $\{ a,b,c,d\}$ & $\{ a, d\}$ & \{ a \} \\
            $b$ & $\{ b\}$ & $\{ b \}$ & \{ b \} \\
            $c$ & $\{ c, d\}$ & $\{ c \}$ & \{ c, d \} \\
            $a$ & $\{ d\}$ & $\{ d \}$ & \{ d \}
        \end{tabular}
    \end{minipage}
    
    \caption{A target set of elementary closed sets $S^*$ and two candidate closure operators}
    \label{fig:closurecomparison}
\end{figure}

Moreover, pretopology is made in such a way that a single collection
of closed sets can be obtained by many different pseudo-closure operators.
For example, there are three ways to obtains a closure $F(\{b\}) = \{ b, c, d \}$
\begin{enumerate}
 	\item $a(\{b\}) = \{ b, c, d \}$
    \item $a(\{b\}) = \{ b, c \}$ and $a(\{ b, c \}) = \{ b, c, d \}$
    \item $a(\{b\}) = \{ b, d \}$ and $a(\{ b, d \}) = \{ b, c, d \}$    
 \end{enumerate}
 It means that when we learn a pretopological space based on its elementary closed sets, we must
 consider the set of all the pseudo-closure operator leading to the same elementary closed sets.
 
 So, given a closure function $F^*(.)$, we learn an underlying 
 pseudo-closure operator, which appears not to be unique. 
 Reworded, it means that we learn a target function which can take
 multiple \emph{shapes}, each described by a different set of features (i.e. different combinations of neighborhoods). This is the formulation of the multiple instance problem 
 \citep{DBLP:journals/ai/DietterichLL97} on which the remaining of the contribution is based. 
    \section{A multiple instance approach}\label{sec:mi}

The multiple instance (MI) problem \citep{DBLP:journals/ai/DietterichLL97} arises when an observation can be described by multiple feature vectors. The set of vectors describing an observation is called a \emph{bag} of instances and is labeled positively or negatively. A bag is labeled according to the chosen MI assumption \citep{DBLP:journals/ker/FouldsF10}, but our work stays within the \emph{standard MI assumption} where a bag is positive if and only if at least one of its instances is positive. On the contrary, a bag is labeled negative when all of its instances are negative.

The MI task consists in finding a function predicting the label of a new instance with the sole knowledge of the bags labels. A famous example of MI task is the \emph{simple jailer problem} \citep{DBLP:conf/ai/ChevaleyreZ01}: given a locked door and a bunch of key rings, the task is to find which kind of key opens it. But the only information we have is if a key ring is useful or not, that is to say if it contains at least one key capable of unlocking the door. An example of a simple jailer dataset is provided in Table \ref{tab:multidata}. Recall that in real world cases the column "Instance labels" is unknown and is represented here in order to understand why Bags 1 and 2 are positives (because they have at least one positive instance). On the contrary, Bag 3 is labeled negatively because it has no positive instances.

\begin{table}
	\centering
    \begin{tabular}{c|c|c|c|c}
       	Bag & Shape & Size & Instance labels & Bag labels \\
        \hline

        \multirow{3}{*}{1} & Squared & Big & 1 & \multirow{3}{*}{1} \\
        & Squared & Small & 0 & \\
        & Squared & Medium & 1 & \\
        \hline

        \multirow{2}{*}{2} & Squared & Medium & 1 & \multirow{2}{*}{1} \\
        & Triangular & Small & 0 & \\
        \hline

        \multirow{3}{*}{3} & Rectangular & Big & 0 & \multirow{3}{*}{0} \\
        & Squared & Small & 0 & \\
        & Triangular & Small & 0 & \\
    \end{tabular}
    
    \caption{Example of a multiple instance dataset on the simple jailer problem}
    \label{tab:multidata}
\end{table}

We model the task of learning a pretopological space based on its expected elementary closed sets $S^*$ by a MI learning task. An instance represents a propagation from a set $A$ to an element $x$ and is labeled positively if $x \in F^*(A)$ according to $S^*$, otherwise the instance is labeled negatively. As we briefly saw previously, multiple propagations for a set $A$ are compliant with $S^*$. Thus we model a bag by the set of instances modeling those propagations.

\subsection{Building a MI dataset for the LPS task}

Given a finite non empty set of elements $E$, we learn a propagation concept which takes the form of a boolean function $Q : {\cal P}(E) \times E \rightarrow \{0, 1\}$. A pseudo-closure operator $a_Q(.)$ is then derived from $Q$ and defined as $\forall A \in {\cal P}(E),\, a_Q(A) = \{ x \in E \mid Q(A, x) = 1 \}$. Finally, $a_Q(.)$ must be designed such that the resulting elementary closed sets $S_Q$ best fits with a target set of elementary closed sets $S^*$.

To solve this problem we first define a method to build a MI dataset from a set of elementary closed sets $S^*$ and a collection of $k$ neighborhoods ${\cal V} = \{ V_1, \ldots, V_k \}$. We describe an observation (a bag) as a set of propagations that must occur or not, depending if we describe a positive or negative bag. If the bag is positive, at least one propagation must occur to be able to retrieve the target pretopological space, and vice versa. A propagation (an instance) is identified by a couple $(A,x) \in {\cal P}(E) \times E$ and is described by the boolean feature vector $(q_1(A, x), \ldots, q_k(A, x))$ with $q_i$ the predicate derived from $V_i \in {\cal V}$ such that $q_i(A, x) \equiv V_i(x) \cap A \neq \emptyset$.

In order to compute the whole set of positive bags required to solve the LPS task, all the sets $A \in {\cal P}(E)$ such that $A$ is expanded by the target pseudo-closure operator must be considered. Since the only information provided to the LPS task is the set $S^*$  of elementary closed sets, we will consider only the sets $A \in {\cal P}(E)$ such that $A$ is reachable (through the pseudo-closure operator) from any $x \in E$. This is possible thanks to the isotonic property of the V-type pretopological spaces.

For all $x \in E$, its elementary closure $F^*(\{ x \})$ is provided to the LPS task. The set of sets reachable from $x$ is deduced from this knowledge : the V-type property guaranties us that any set $A$ such that $\{ x \} \subseteq A$ is expanded to, a least, all the elements in $F^*(\{ x \})$ ($F^*(\{ x \}) \subseteq F^*(A)$). However, any set $A$ such that $A \subseteq F^*(\{ x \}) \neq \emptyset$ can be safely ignored since such a set cannot be reached from $x$.
As a consequence, only the sets $A$ such that $\{ x \} \subseteq A \subset F^*(\{ x \})$ should be considered to build a positive bag engendered by $x$. The set $F^*(\{ x \})$ is excluded because, by definition of a closed set, it cannot be expanded by the pseudo-closure operator. For any set $A$ fulfilling these criteria, the V-type property tells that $A$ must be expanded, through the closure operator $F^*(.)$, to exactly all the elements in $F^*(\{ x \})$. That is to say, $A$ must be expanded, through the pseudo-closure operator, to a least one element in $F^*(\{ x \}) \setminus A$.
Hence, for all sets $A$ such that $\{ x \} \subseteq A \subset F^*(\{ x \})$, $x$ engenders the positive bag $bag^+(x, A)$ which models the propagation from $A$ to an element in $F^*(\{ x \}) \setminus A$. The whole set of positive bags engendered by $x$ is noted $bags^+(x)$.

\begin{align}
	\forall x \in E,\, \forall \{ x \} \subseteq A \subset F^*(\{ x \}),~ &bag^+(x, A) = \{ (A, y) \}_{y \in F^*(\{ x \}) \setminus A} \\
    \forall x \in E,~ &bags^+(x) = \{ bag^+(x, A) \}_{\{ x \} \subseteq A \subset F^*(\{ x \})}
\end{align}

It should be mentioned that, in some cases where two (or more) elements $x,y \in E$ share the same elementary closure $F^*_{xy} = F^*(\{ x \}) = F^*(\{ y \})$, some bags are engendered by both $x$ and $y$. Actually, any positive bag $bag^+(x, A)$ where $\{ x, y \} \subseteq A \subset F^*_{xy}$ is engendered by both $x$ and $y$. Such a positive bag will be designated as either $bag^+(x, A)$, $bag^+(y, A)$ or $bag^+(\{ x, y \}, A)$.
This is not a trivial property since it cause some trouble evaluating a potential solution.

The computation of the whole set of negative bags required to solve the LPS task is a bit different. For all $x \in E$, the set of elements not reachable by $x$ must be computed. It is an easy task since is corresponds to the elements in $E \setminus F^*(\{ x \})$. For each $y \in E \setminus F^*(\{ x \})$, the sets that are not expanded to $y$ are determined: for all $A\in {\cal P}(E)$, if $\{ x \} \subseteq A \subseteq F^*(\{ x \})$ then $A$ is not expanded to $y$ ($y \notin F^*(A)$). Notice that this time $F^*(\{ x \})$ is not ignored because it should \emph{not} be expanded to any element.  If there is an $A$ such that $y \in F^*(A)$ then it means $\{ y \} \subseteq F^*(\{ x \}) \subseteq F^*(A)$. This is impossible since $S^*$ tells us that $y \notin F^*(\{ x \})$. Hence, for all elements $y \in E \setminus F^*(\{ x \})$, $x$ engenders a negative bag $bag^-(x, y)$ which models the propagations forbidden by $S^*$. The whole set of negative bags engendered by $x$ is noted $bags^-(x)$. 

\begin{align}
	\forall x \in E,\, \forall y \in E \setminus F^*(\{ x \}),~  &bag^-(x, y) = \{ (A, y) \}_{\{ x \} \subseteq A \subseteq F^*(\{x\})} \\
    \forall x \in E,~ &bags^-(x) = \{ bag^-(x, y) \}_{y \in E \setminus F^*(\{ x \})}
\end{align}

It is not possible for a negative bag to be shared by multiple elements. Indeed, a negative bag engendered by $x \in E$ contains an instance $(\{ x \}, y)$ describing the expansion of the singleton $\{ x \}$ to an element $y \notin F^*(\{ x \})$. It is therefore impossible for the instance $(\{ x \}, y)$ to belong to a negative bag engendered by  $z \neq x$.

For each $x \in E$, a positive bag $bag^+(x, A)$ is engendered (by $x$) for all $A$ where $A$ is framed between a bottom element $\bot \in {\cal P}(E)$ and a top element $\top \in {\cal P}(E)$: $\bot \subseteq A \subset \top$. This property allows for a good visual representation of the positive bags engendered by $x$. Let's consider the lattice $\cal L$ of the elements in $E$, the set of positive bags $bags^+(x)$ engendered by $x \in E$ is visually represented by a subset of $\cal L$ whose bottom element is $\{ x \}$ and whose top element is $F^*(\{ x \})$ (excluded). We define two sub-lattice notations, inspired by the notation of intervals, defined for all $A \subseteq B \subseteq E$.

\begin{align}
	\sublat{A}{B} &= {\cal F}_A \cap {\cal P}(B) \\
    \sublatx{B}{B} &= \sublat{A}{B} \setminus B
\end{align}

$\sublat{A}{B}$ is defined as the intersection between the supersets of $A$ ($= {\cal F}_A$) and the subsets of $B$ ($= {\cal P}(B)$), that is to says the sets between $A$ and $B$ in the lattice $\cal L$. $\sublatx{A}{B}$ is equal to $\sublat{A}{B}$ depleted of its biggest element.

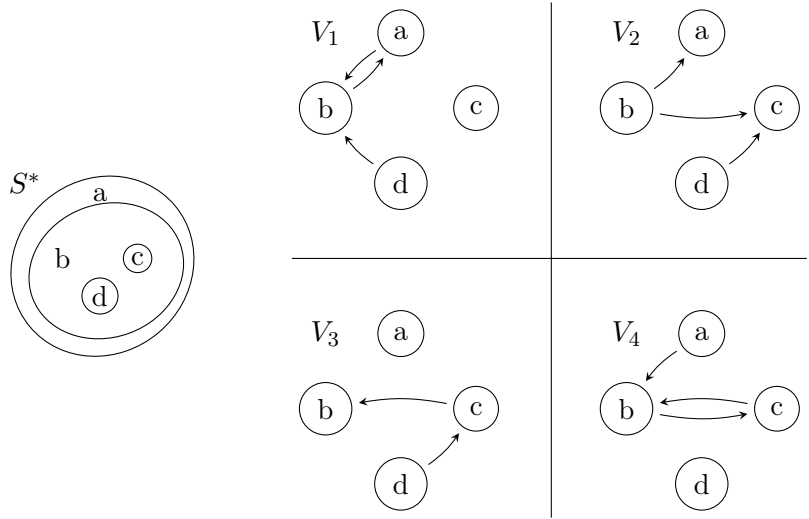
\begin{figure}
    \centering

    \begin{tikzpicture}[graph_name/.style={}, graph_node/.style={circle, draw, outer sep=1mm}, >=stealth]
        \coordinate (xshift) at (4, 0);
        \coordinate (yshift) at (0, -4);

        \node (ref) at (0, 2) {$S^*$};
        \node (refa) at (1, 1.85) {a};
        \node (refb) at (0.5, 1) {b};
        \node[circle, draw, inner sep=0.5mm, outer sep=1mm] (refc) at (1.5, 1) {c};
        \node[circle, draw, inner sep=0.5mm, outer sep=1mm] (refd) at (1, 0.5) {d};

        \node[ellipse, inner sep=0.2mm, draw, rotate fit=20, fit={(refb) (refc) (refd)}] (Fb) {};
		\node[ellipse, inner sep=0mm, draw, rotate fit=40, inner sep=-1mm, fit={(Fb) (refa)}] {};

        \node[graph_name] (V1) at ($(ref) + (xshift) + (0, 2)$) {$V_1$};
        \node[graph_node] (a1) at ($(1, 4) + (xshift)$) {a};
        \node[graph_node] (b1) at ($(0, 3) + (xshift)$) {b};
        \node[graph_node] (c1) at ($(2, 3) + (xshift)$) {c};
        \node[graph_node] (d1) at ($(1, 2) + (xshift)$) {d};
        \path[->] 
            (a1) edge[bend right=10] (b1) 
            (b1) edge[bend right=10] (a1) 
            (d1) edge[bend left=10] (b1);

        \node[graph_name] (V2) at ($(V1) + (xshift)$) {$V_2$};
        \node[graph_node] (a2) at ($(a1) + (xshift)$) {a};
        \node[graph_node] (b2) at ($(b1) + (xshift)$) {b};
        \node[graph_node] (c2) at ($(c1) + (xshift)$) {c};
        \node[graph_node] (d2) at ($(d1) + (xshift)$) {d};
        \draw[->] 
            (b2) edge[bend right=10] (a2) 
            (b2) edge[bend right=10] (c2)
            (d2) edge[bend right=10] (c2);


        \node[graph_name] (V3) at ($(V1) + (yshift)$) {$V_3$};
        \node[graph_node] (a3) at ($(a1) + (yshift)$) {a};
        \node[graph_node] (b3) at ($(b1) + (yshift)$) {b};
        \node[graph_node] (c3) at ($(c1) + (yshift)$) {c};
        \node[graph_node] (d3) at ($(d1) + (yshift)$) {d};
        \path[->] 
            (d3) edge[bend right=10] (c3)
            (c3) edge[bend right=10] (b3);

        \node[graph_name] (V4) at ($(V1) + (xshift) + (yshift)$) {$V_4$};
        \node[graph_node] (a4) at ($(a1) + (xshift) + (yshift)$) {a};
        \node[graph_node] (b4) at ($(b1) + (xshift) + (yshift)$) {b};
        \node[graph_node] (c4) at ($(c1) + (xshift) + (yshift)$) {c};
        \node[graph_node] (d4) at ($(d1) + (xshift) + (yshift)$) {d};
        \path[->] 
            (a4) edge[bend right=10] (b4) 
            (b4) edge[bend right=10] (c4)
            (c4) edge[bend right=10] (b4);

        \draw (a1.north) ++ (2, 0) -- ($(d3.south) + (2, 0)$);
        \draw (b1.west) ++ (0, -2) -- ($(c2.east) - (0, 2)$);
    \end{tikzpicture}

    \caption{A set of elementary closures $S^*$ and four neighborhoods}
    \label{fig:buildingdataset}
\end{figure}

To illustrate this bags generation process, we consider the target set of elementary closed sets $S^*$ and the collection of four neighborhoods showed in Figure \ref{fig:buildingdataset}. Let's consider the element $b$: according to $S^*$, we know that $F^*(\{b\}) = \{ b, c, d \}$. We deduce from this statement that there are three valid V-type pseudo-closure operators, corresponding to the three following behaviors : 
\begin{enumerate}
	\item $a_{Q_1}(\{b\}) = \{b, c, d\}$ and $a_{Q_1}(\{b, c, d\}) = \{b, c, d\}$
	\item $a_{Q_2}(\{b\}) = \{b, c\}$ then $a_{Q_2}(\{b, c\}) = \{b, c, d\}$ and $a_{Q_2}(\{b, c, d\}) = \{b, c, d\}$
	\item $a_{Q_3}(\{b\}) = \{b, d\}$ then $a_{Q_3}(\{b, d\}) = \{b, c, d\}$ and $a_{Q_3}(\{b, c, d\}) = \{b, c, d\}$
\end{enumerate}
Thus, we know the general behavior the learned pseudo-closure operator must respect with the sets in $\sublat{\{ b \}}{F^*(\{ b \})}$.
\begin{enumerate}
	\item $\{b\}$ must propagate to at least $c$ or $d$
	\item $\{b, c\}$ must propagate to $d$, since $d \in F^*(\{b\})$ and $\{b\} \subseteq \{b, c\}$ (V-type property)
	\item $\{b, d\}$ must propagate to $c$, since $c \in F^*(\{b\})$ and $\{b\} \subseteq \{b, d\}$ (V-type property)
	\item $\{b\}$, $\{b, c\}$, $\{b, d\}$ and $\{b, c, d\}$ must not propagate to $a$
\end{enumerate}

The bags are built according to these four behaviors. To model a "must propagate" behavior, a positive bag is built, and to model a "must not propagate" behavior, a negative bag is built. We know that any set $A \in {\cal P}(E)$ such that $b \in A$ must propagate to $c$ or $d$ (or both). For example, the set $\{a, b \}$ must propagate to $c$ or $d$. But $S^*$ informs us that the element $a$ must not appear in the closure of $\{b\}$. Since we consider only the elementary closed sets, and since $\{a, b\}$ is not reachable from $\{ b \}$, no information about the expansion of $\{ a, b \}$ is useful to learn the closure of $\{ b \}$. As a consequence, only the supersets of $\{b\}$ strictly included in $F^*(\{b\})$ are considered,  that is to says the sub-lattice $\sublatx{\{b\}}{F^*(\{b\})} = \{ \{ b \}, \{ b, c \},\{ b, d \} \}$. Figure \ref{fig:lattice} shows the lattice $\cal L$ and highlights the sets describing the positive bags engendered by $b$ ($bags^+(b)$).

Let's build the positive bags engendered by $b$. The first bag, $bag^+(b, \{b\})$, contains two instances: $(\{b\},c)$ and $(\{b\},d)$. Two other positive bags are built: $bag^+(b, \{ b, c\})$ and $bag^+(b, \{ b, d \})$ contain both one instance, respectively, $(\{b, c\},d)$ and $(\{b, d\},c)$.

Building the negative bags is a little bit different. We know that each set belonging to $\sublat{\{b\}}{F^*(\{b\})} = \{ \{b\}, \{b, c\}, \{b, d\}, \{b, c, d\} \}$ must not propagate to $a$. So a single negative bag is built: $bag^-(b, a)$ contains the four negative instances $(\{b\},a)$, $(\{b, c\},a)$, $(\{b, d\},a)$ and $(\{b, c, d\},a)$.

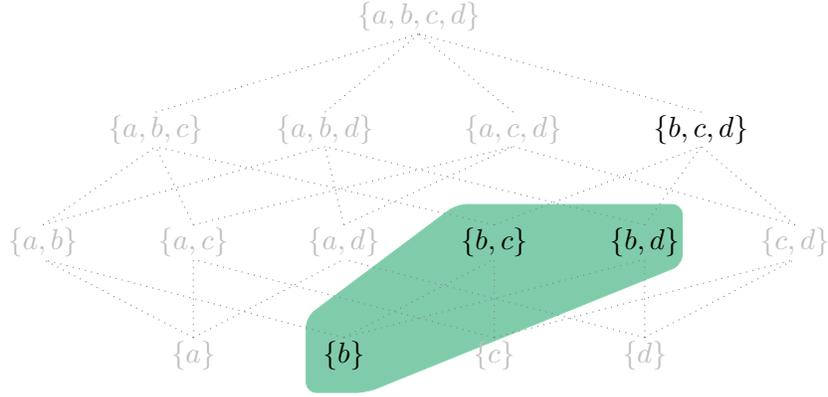
\begin{figure}
	\centering
    \begin{tikzpicture}[
        every node/.style={draw=none, inner sep=0.3mm, transform shape},
        edge/.style={thin, color=gray, dotted},
        shaded/.style={thin, color=lightgray},
        covered/.style={draw, shape=rectangle},
        rejected/.style={draw, shape=strike out},
        >=stealth,
        scale=1
    ]
        \node[shaded] (a) at (0, 0) {$\{ a \}$};
        \node (b) at (2, 0) {$\{ b \}$};
        \node[shaded] (c) at (4, 0) {$\{ c \}$};
        \node[shaded] (d) at (6, 0) {$\{ d \}$};

        \node[shaded] (ab) at (-2, 1.5) {$\{ a, b \}$};
        \node[shaded] (ac) at (0, 1.5) {$\{ a, c \}$};
        \node[shaded] (ad) at (2, 1.5) {$\{ a, d \}$};
        \node (bc) at (4, 1.5) {$\{ b, c \}$};
        \node (bd) at (6, 1.5) {$\{ b, d \}$};
        \node[shaded] (cd) at (8, 1.5) {$\{ c, d \}$};

        \node[shaded] (abc) at (-0.5, 3) {$\{ a, b, c \}$};
        \node[shaded] (abd) at (1.75, 3) {$\{ a, b, d \}$};
        \node[shaded] (acd) at (4.25, 3) {$\{ a, c, d \}$};
        \node (bcd) at (6.75, 3) {$\{ b, c, d \}$};
        
        \node[shaded] (abcd) at (3, 4.5) {$\{ a, b, c, d \}$};

        \draw[edge] (a.north) edge (ab.south) edge (ac.south) edge (ad.south);
        \draw[edge] (b.north) edge (ab.south) edge (bc.south) edge (bd.south);
        \draw[edge] (c.north) edge (ac.south) edge (bc.south) edge (cd.south);
        \draw[edge] (d.north) edge (ad.south) edge (bd.south) edge (cd.south);

        \draw[edge] (ab.north) edge (abc.south) edge (abd.south);
        \draw[edge] (ac.north) edge (abc.south) edge (acd.south);
        \draw[edge] (ad.north) edge (abd.south) edge (acd.south);
        \draw[edge] (bc.north) edge (abc.south) edge (bcd.south);
        \draw[edge] (bd.north) edge (abd.south) edge (bcd.south);
        \draw[edge] (cd.north) edge (acd.south) edge (bcd.south);
        
        \draw[edge] (abc.north) edge (abcd.south);
        \draw[edge] (abd.north) edge (abcd.south);
        \draw[edge] (acd.north) edge (abcd.south);
        \draw[edge] (bcd.north) edge (abcd.south);

         
         \begin{pgfonlayer}{background}
             \coordinate (pos0) at ($(b) + (-0.5, -0.5)$) ;
             \coordinate (pos1) at ($(b) + (-0.5, 0.5)$) ;
             \coordinate (pos2) at ($(bc) + (-0.5, 0.5)$) ;
             \coordinate (pos3) at ($(bd) + (0.5, 0.5)$) ;
             \coordinate (pos4) at ($(bd) + (0.5, -0.3)$) ;
             \coordinate (pos5) at ($(b) + (0.3, -0.5)$) ;
             \draw[Green, fill=ForestGreen, opacity=0.5, rounded corners] (pos0) -- (pos1) -- (pos2) -- (pos3) -- (pos4) -- (pos5) -- cycle ;
         \end{pgfonlayer}
    \end{tikzpicture}
    
    \caption{The lattice $\cal L$ of the elements in $E = \{ a, b, c, d \}$.
    The sets in $\sublat{\{ b \}}{F^*(\{b\})}$ are displayed in black and are useful to learn the propagation of $\{ b \}$, either because they must or must not be expanded by the learned pseudo-closure operator.
    Each set $A$ in the green sub-lattice ${\cal L} \left[ \{b\}, F^* \left( \{b\} \right) \right[$ gives rise to a positive bag engendered by $b$: $b^+(b, A)$.}
    \label{fig:lattice}
\end{figure}

\begin{table}
    \centering
    \resizebox{\linewidth}{!}{%
    \begin{tabular}{c|c|cccc|c|c}
        Id & $(A,x) \in {\cal P}(E) \times E$ & $q_1(A, x)$ & $q_2(A,x)$ & $q_3(A,x)$ & $q_4(A,x)$ & $x \in a(A)$ & $x \in F^*(A)$ \\ \hline
        \multirow{2}{*}{$bag^+(b, \{b\})$} & $(\{b\},c)$                      & 0           & 0          & 1          & 1          & ?            & \multirow{2}{*}{1}            \\
        & $(\{b\},d)$                      & 1           & 0          & 0          & 0          & ?            &            \\ \hdashline
        $bag^+(b, \{b, c\})$ & $(\{b, c\},d)$                   & 1           & 1          & 1          & 0          & 1            & 1            \\ \hdashline
        $bag^+(b, \{b, d\})$ & $(\{b, d\},c)$                   & 0           & 0          & 1          & 1          & 1            & 1            \\ \hdashline
        \multirow{4}{*}{$bag^-(b, a)$} & $(\{b\},a)$                      & 1           & 0          & 0          & 1          & 0            & \multirow{4}{*}{0}            \\
        & $(\{b, c\},a)$                   & 1           & 0          & 0          & 1          & 0            &           \\
        & $(\{b, d\},a)$                   & 1           & 0          & 0          & 1          & 0            &             \\
        & $(\{b, c, d\},a)$                & 1           & 0          & 0          & 1          & 0            &            
    \end{tabular}%
    }
    
    \caption{Bags engendered by $b$}
    \label{tab:bagsengendered}
\end{table}

Table \ref{tab:bagsengendered} shows the four bags engendered by the element $b$. The values in the column "$x \in a(A)$" are the instance labels we try to learn. The interrogation marks in the two first lines show that we do not know, and actually do not care, whether $b$ must propagate to $c$, $d$ or both. Though, since they belong to a positive bag, we know that at least one instance is positive, that is to say $c \in a(\{b\}) \vee d \in a(\{b\})$. 
We are able to deduce how the other instances are labeled since they belong either to a positive bag compounded of a single instance or to a negative bag: they are labeled as positive in the first case and as negative in the second.

\subsection{The issue of the exponential MI dataset}

As exemplified above, the Multiple Instance framework offers an elegant and accurate modeling for the not trivial pretopology learning problem. However, since pretopology refers to the powerset ${\cal P}(E)$, the number of required positive bags engendered by an element $x \in E$ is exponential (proportional to the size of its elementary closure $F^*(\{x\})$) thus revealing a major issue for the learning process.

We just show how we build a dataset based on a set of elementary closed sets. With such a dataset, we could apply standard MI algorithms to learn a solution from it. A standard MI approach would require the enumeration of every instances in order to count the number of covered bags, this will not be a problem with the example given above. But as soon as we work on a real problem the number of generated positive bags is overwhelming. Indeed, we explain that $\forall x \in E$, we take all the sets belonging to the sub-lattice ${\cal L} \left[ \{x\}, F^*(\{x\}) \right[$ into account. Such a number of bags cannot be treated efficiently by a standard learning algorithm, since it is exponential in the size of $F^*(\{x\})$.

A standard greedy MI algorithm would collect, for each learning step, the set of covered  positive bags and remove them so that the following step would focus on the remaining bags only ; and so on until any positive bag is covered (or another termination criterion is reached). We remark that such an algorithm does not actually care to know whether a given positive bag is covered or not. It rather cares to know how many of them are covered, and if the termination criterion is reached.

Thus we propose a method to count (or estimate) the number of bags covered by a solution. This method relies on a trick relying itself on the properties of V-type pretopological spaces. Counting the number of covered negative bags is not a big deal so we will not cover it deeply. However, counting the number of covered positive bags is a problem requiring much more efforts. We show that while we are able to count the number of total positive bags, it is actually inefficient to count precisely the number of positive bags covered by a solution. Thus we estimate this number by subtracting a (high) estimation of
the positive bags rejected (i.e. not covered) by the solution to the number of total generated positive bags in our dataset.
%
%
%
%



\subsection{Estimate of true/false positives}

In the following a method is presented to estimate the number of positive/negative bags covered by a solution. We start by defining the total number of positive/negative bags in a dataset stemming from the LPS task. Then we detail the method for estimating the number of positive bags covered by a solution. Finally, the computation of the covered negative bags is briefly presented.

\subsubsection{Computing the total number of positive/negative bags}

The number of positive bags engendered by an element $x \in E$ (noted $b^+(x)$) matches with the number of elements in the sub-lattice $\sublatx{\{x\}}{F^*(\{x\})}$:
\begin{equation}
	b^+(x)=2^{|F^*(\{x\})| - 1} - 1
\end{equation}
But we cannot just sum the number of engendered positive bags for each $x \in E$ since, in some cases where multiple elements share the same elementary closure, a bag is engendered by multiple elements. For example, if $F^*(\{a\}) = F^*(\{b\}) = \{ a, b, c \}$, then 
the positive bag identified by $(a, \{ a, b \})$ is engendered by both elements $a$ and $b$ ; we remark that this bag is then identified by $(b, \{ a, b \})$ too.

Considering the set $\{F^*_1,\dots ,F^*_K\}$ of the $K$ distinct elementary closed sets from $E$ such that $\forall x \in E, \exists k \in \{1..K\}$ such that $F^*(\{x\})=F^*_k$ ; $E$ can be partitioned into a set of $K$ equivalence classes ${\cal A} = \{A_1,\dots A_K\}$ where each class $A_k$ is composed of the elements whose closure is $F^*_k$ ($\forall x \in A_k, F^*(\{x\})=F^*_k$). 

For each equivalence class $A_k \in {\cal A}$, the number of positive bags engendered by the whole subset $A_k$ is computed using the \emph{inclusion-exclusion} principle.
\begin{equation}
    \begin{split}
        \forall A_k \in {\cal A}, ~~~b^+(A_k) 
        &= \sum_{i=1}^{|A_k|} (-1)^{i+1} \sum_{X \in comb(A_k, i)} \left| \bigcap_{x \in X} \sublatx{\{x\}}{F^*_k} \right| \\
        &= \sum_{i=1}^{|A_k|} (-1)^{i+1} \binom{|A_k|}{i} (2^{|F^*_k|-i} - 1)
    \end{split}
\end{equation}
Where $comb(A_k, i)$ expresses the power set of $A_k$ reduced to its elements of size $i$. So we alternatively add and subtract the size of the intersection between any sub-lattice $\sublatx{X}{A_k}$ where $|X| = i$ and $X \subseteq A_k$. Since all considered sub-lattices share the same biggest element $F^*_k$, they also share the same size, which is $2^{|F^*_k| - i} - 1$ thus simplifying the expression of $b^+(A_k)$.




\begin{property}\label{prop:latticeinter}
	Let $(E, a)$ a V-type pretopological space, $A,B \in {\cal P}(E)$, the sub-lattice $\sublat{A}{F(A)}$ intersects $\sublat{B}{F(B)}$ if and only if $A$ and $B$ share the same closure.
    \begin{equation}
    	\forall A \in {\cal P}(E), \forall B \in {\cal P}(E),\; 
        \sublat{A}{F(A)} \cap \sublat{B}{F(B)} \neq \emptyset \Leftrightarrow F(A) = F(B)
    \end{equation}
\end{property}

\begin{proof}
	Consider a V-type pretopological space $(E, a)$ and two sets $A \in {\cal P}(E)$ and $B \in {\cal P}(E)$.
    
    \begin{itemize}
        \item 
        	If $F(A) = F(B) = K$, then $\sublat{A}{F(A)}$ and $\sublat{B}{F(B)}$ share at least their biggest element $K$. Hence $F(A) = F(B) \Rightarrow \sublat{A}{F(A)} \cap \sublat{B}{F(B)} \neq \emptyset$.

        \item 
        	If $\sublat{A}{F(A)} \cap \sublat{B}{F(B)} \neq \emptyset$, then $\exists C \in {\cal P}(E)$ such that $A \subseteq C$, $B \subseteq C$, $C \subseteq F(A)$ and $C \subseteq F(B)$. Then, by definition of a V-type pretopological space, $F(A) = F(C)$:
            \begin{align*}
                A \subseteq C \subseteq F(A) &\Rightarrow F(A) \subseteq F(C) \subseteq F(A) \\
                &\Rightarrow F(A) = F(C)
            \end{align*}
            We can show similarly that $F(B) = F(C)$. Hence, $\sublat{A}{F(A)} \cap \sublat{B}{F(B)} \neq \emptyset \Rightarrow F(A) = F(B)$.
    \end{itemize}
\end{proof}

Property \ref{prop:latticeinter} ensures us that the total number ($B^+$) of positive bags can be computed by simply summing the number of positive bags engendered for each equivalence class:
\begin{equation}
	B^+=\sum_{A_k \in {\cal A}} b^+(A_k)
\end{equation}

The number of negative bags is much more simpler to calculate since it does not require to rely on the inclusion-exclusion principle. The number of negative bags engendered by an element $x \in E$ (noted $b^-(x)$) matches the size of $\overline{F^*(\{x\})}$:

\begin{equation}
	b^-(x) = |E \setminus F^*(\{x\})|
\end{equation}
Hence the total number ($B^-$) of negative bags is obtained by summing the negative bags engendered by each element 
\begin{equation}
	B^-=\sum_{x \in E} b^-(x)
\end{equation}

We are now able to calculate the total number of positive and negative bags, without explicitly generating them. In order to know the number of positive bags covered by a solution $Q$ under construction, we need a way to determine how many positive bags are {\bfseries not} covered by $Q$. Then we will get the number of covered positive bags by subtracting the number of not covered positive bags to the total number of positive bags.
We show in the following a method to estimate the positive bags not covered by $Q$ and a method to calculate the number of negative bags covered by $Q$.

\subsubsection{Number of positive bags covered by a solution}

We estimate the number of positive bags covered (the true positives) by a solution $Q$ by subtracting the (estimated) number of not yet covered positive bags to the total number of positive bags. Due to the complexity of the pseudo-closure operator, we cannot count precisely the number of not covered positive bags. We propose a method to estimate this number.

We consider a set of target elementary closed sets $S^*$ and a solution $Q$ under construction. For any element $x \in E$, we consider the true part of its learned closure, noted $F_Q^*(\{x\}) = F_Q(\{x\}) \cap F^*(\{x\})$. The definition of the V-type pretopological space $(E,a_Q)$ guaranties us that $\forall A \in {\cal P}(E),\, x \in A \Rightarrow F_Q^*(\{x\}) \subseteq F_Q(A)$\footnote{since $F_Q^*(\{x\}) \subseteq F_Q(\{x\})$ and $F_Q(\{x\}) \subseteq F_Q(A)$ by isotony.}. It means that any set containing $x$ will be properly expanded by $(E,a_Q)$ to at least $F_Q^*(\{x\})$.
If we consider only the informations provided by $F_Q^*(\{x\})$, it also means that $F_Q^*(\{x\})$ is not properly expanded, so the positive bag $bag^+(x, F_Q^*(\{x\}))$ is not covered by $Q$. In the following, a (strong) assumption is made  that any positive bag $bag^+(x, A)$ where $F_Q^*(\{x\}) \subseteq A \subset F^*(\{x\})$\footnote{The sub-lattice $\sublatx{F_Q^*(\{x\})}{F^*(\{x\})}$} is not covered by $Q$ : it is called the \emph{elementary coverage assumption}.
%
%
%
%
%
%
%
%
%
%
%
%

\paragraph{The elementary coverage assumption}{
	For all $x \in E$, the set of positive bags whose coverage (by $Q$) is \emph{not} ensured by the learned closure of $x$ ($F_Q(\{ x \})$) is defined as:
    \begin{equation}
    	ec(x) = \sublatx{F_Q^*(\{ x \})}{F^*(\{ x \})}
    \end{equation}
    $ec(x)$ is the set of positive bags engendered by $x$ which are rejected by $Q$ if the sole knowledge provided by $F_Q^*(\{x\})$ is considered. It is a high estimation because it is possible that $\exists A,\, F_Q^*(\{x\}) \subset A \subset F^*(\{x\})$ such that $A$ is properly propagated by $a_Q(.)$. When such a case occurs, it means that $bag^+(x, A)$ is covered by $Q$ while it is not detected as such. A way to overcome this issue is to consider the supersets of $F_Q^*(\{x\})$, but it would require to compute the closure of every supersets, which would be terribly inefficient.
%
%
%
%
}


With that in mind, a (lower) estimation of the positive bags whose coverage (by $Q$) is ensured by the closure of $x \in E$ can be computed by subtracting the size of $ec(x)$ to the total number of positive bags engendered by $x$.
%
%
%
%
%
%

Figure \ref{fig:coveredposbagsQ} illustrates how theses covered positive bags are estimated. Let $E = \{ a, b, c, d\}$, $F^*(\{a\}) = \{a, b, c, d\}$ and $Q$ such that $F_Q(\{a\}) = \{ a, b \}$. We consider the whole set of positive bags engendered by the element $a$: $\sublatx{\{a\}}{F^*(\{a\})}$. The smallest set that is not properly expanded by $a_Q(.)$ is $F_Q^*(\{a\}) = \{ a, b \}$, so we suppose that every set above $\{ a, b \}$ (i.e. $ec(a)$) are not properly expanded too. This appears to be wrong since $F_Q(\{a, b, d\} = F^*(\{a\})$ so the bag $bag^+(a, \{ a, b, d \})$ is covered by $Q$ but ignored by our estimation.
The number of positive bags whose coverage is ensured by $F_Q(\{ x \})$ (the green area) is then calculated by subtracting the size of $ec(a)$ to the total number of positive bags engendered by $a$ : $\left| \sublatx{\{a\}}{F^*(\{a\})} \right| - |ec(a)| = 4$.
%
%

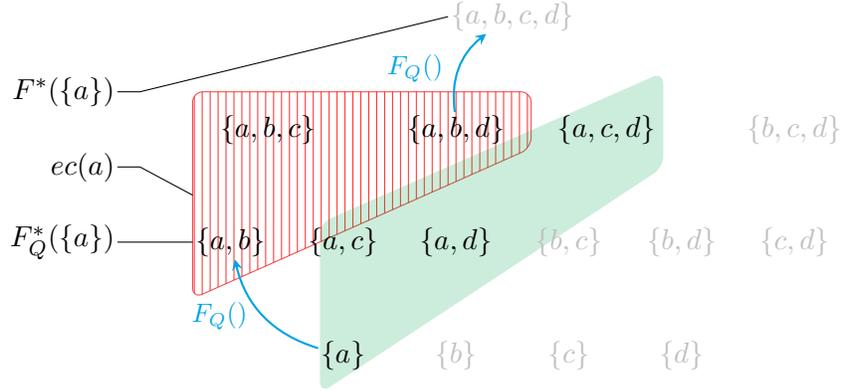
\begin{figure}
	\centering
    \begin{tikzpicture}[
        every node/.style={draw=none, inner sep=0.3mm},
        edge/.style={thin, color=gray, dotted},
        shaded/.style={thin, color=lightgray},
        scale=1
    ]
    
        \node (a) at (0, 0) {$\{ a \}$};
        \node[shaded] (b) at (1.5, 0) {$\{ b \}$};
        \node[shaded] (c) at (3, 0) {$\{ c \}$};
        \node[shaded] (d) at (4.5, 0) {$\{ d \}$};

        \node (ab) at (-1.5, 1.5) {$\{ a, b \}$};
        \node (ac) at (0, 1.5) {$\{ a, c \}$};
        \node (ad) at (1.5, 1.5) {$\{ a, d \}$};
        \node[shaded] (bc) at (3, 1.5) {$\{ b, c \}$};
        \node[shaded] (bd) at (4.5, 1.5) {$\{ b, d \}$};
        \node[shaded] (cd) at (6, 1.5) {$\{ c, d \}$};

        \node (abc) at (-1, 3) {$\{ a, b, c \}$};
        \node (abd) at (1.5, 3) {$\{ a, b, d \}$};
        \node (acd) at (3.5, 3) {$\{ a, c, d \}$};
        \node[shaded] (bcd) at (6, 3) {$\{ b, c, d \}$};
        
        \node[shaded] (abcd) at (2.25, 4.5) {$\{ a, b, c, d \}$};
        


		\draw[>=stealth, Cerulean, thick] (a) edge[->, bend left] node[midway, left, inner xsep=2mm] {\small$F_Q()$} (ab) ;
		\draw[>=stealth, Cerulean, thick] (abd) edge[->, bend left] node[midway, left, inner xsep=2mm] {\small$F_Q()$} (abcd) ;
        
        \begin{pgfonlayer}{background}         
            \coordinate (ec1) at (-2, 0.75) ;
            \coordinate (ec2) at (-2, 3.5) ;
            \coordinate (ec3) at (2.5, 3.5) ;
            \coordinate (ec4) at (2.5, 2.75) ;
            \coordinate (ec) at ($(ec1)!.5!(ec2)$) ;
            
            \draw[Red, fill=Red, pattern color=Red, opacity=0.8, pattern=vertical lines, rounded corners] (ec1) -- (ec2) -- (ec3) -- (ec4) -- cycle ;
            
             
            \coordinate (c1) at (-0.3, -0.5) ;
            \coordinate (c2) at (-0.3, 1.75) ;
            \coordinate (c3) at (4.25, 3.75) ;
            \coordinate (c4) at (4.25, 2.5) ;
            \coordinate (c) at ($(c1)!.5!(c2)$) ;
            
            \draw[Green, fill=Green, opacity=0.2, rounded corners] (c1) -- (c2) -- (c3) -- (c4) -- cycle ;
		\end{pgfonlayer}
        
        \coordinate (hline_len) at (0.3, 0) ;
        
        \node[anchor=east] (labelFQa) at (-3, 1.5) {$F_Q^*(\{ a \})$} ;
        \draw (labelFQa.east) -- +(hline_len) -- (ab.west) ;
        
        \node[anchor=east] (labelFa) at (-3, 3.5) {$F^*(\{ a \})$} ;
        \draw (labelFa.east) -- +(hline_len) -- (abcd.west) ;
        
        
        \node[anchor=east] (labelFQstara) at (-3, 2.5) 
        	{$ec(a)$} ;
        \draw (labelFQstara.east) -- +(hline_len) -- (ec) ;    
        
    \end{tikzpicture}
    
    \caption{Estimation of the positive bags engendered by $a$ that are covered by a solution $Q$. We suppose $F^*(\{a\}) = \{ a, b, c, d \}$ and $F_Q(\{a\}) = \{ a, b \}$. The positive bags whose coverage is ensured by $F_Q(\{ a \})$ are represented by the green area, which is $\sublatx{\{ a \}}{F^*(\{ a \}} \setminus ec(a)$.}
    \label{fig:coveredposbagsQ}
\end{figure}

For all $x \in E$, we are able to estimate (through a lower bound) the number of positive bags engendered by $x$ and covered by $Q$, by subtracting the number of \emph{not} covered bags to the number of bags engendered.
%
%
%
%
But, as for the calculation of the total number of positive bags, we need to be careful when multiple elements share the same elementary closed set. We use the inclusion-exclusion principle again to compute a fair estimation of the positive bags not covered by $Q$.
We define the function $r_Q^+ : \cup_{A_k \in {\cal A}}{\cal P}(A_k) \rightarrow \mathbb{N}$ which, given a subset $B$ of an equivalence class $A_k$, outputs an estimation of the number of positive bags engendered by the elements in $B$ whose coverage by $Q$ is not always guaranteed by all the closures of the elements in $B$.


\medskip

On a general purpose, we first define how to calculate the size of an union of $n$ sub-lattices ${\cal L}_1,\dots ,{\cal L}_n$ having the same top element noted $\top$ and different bottom elements $\bot_1, \dots, \bot_n $. 
\begin{equation}
	\begin{split}
		\left| \bigcup_{i=1}^n {\cal L}_i \right| 
        &= 
        	\sum_{i=1}^{n} (-1)^{i+1} 
            \sum_{B \in comb(\{1 \dots n\}, i)} \left|\bigcap_{j \in B} {\cal L}_j\right| \\
        &= 
        	\sum_{i=1}^{n} (-1)^{i+1} 
            \sum_{B \in comb(\{1 \dots n\}, i)} 2^{|\top| - |\cup_{j \in B}\bot_j|}
	\end{split}
\end{equation}
Where $comb(\{1 \dots n\}, i)$ expresses the power set of $\{1 \dots n\}$ reduced to its elements of size $i$. So the size of the union of sub-lattices is calculated by alternatively adding and subtracting the size of the intersection between the sub-lattices. The intersection between multiple sub-lattices sharing the same top element $\top$ is the sub-lattice whose biggest element is $\top$ (obviously) and whose bottom element is the union of the bottom elements of the considered sub-lattices. Then the size of this intersection is expressed by $2^{|\top| - |\cup_{j \in B}\bot_j|}$.

\medskip

Given any $i$-permutation $B$ of the equivalence class $A_k$, we use the elementary coverage assumption to estimate the number of positive bags whose coverage, by $Q$, is not always ensured by all the closures of the elements in $B$. That is to say, the number of positive bags that belongs to, at least, one element of $\{ ec(x) \}_{x \in B}$.
%
%
%
%
%
%
%
%
\begin{equation}
    \forall B \in comb(A_k, i),\, r_Q^+(B)
    = \left| \bigcup_{x \in B} \sublat{F_Q^*(\{x\}) \cup B}{F_k^*} \right| - 1
\end{equation}
We must subtract 1 in order to ignore the top element $F_k^*$, because $(\_, F_K^*)$ is not a positive bag.
%
%
%
%

Then, for any equivalence class $A_k \in {\cal A}$, the estimated number of positive bags engendered by (the elements of) $A_k$ and covered by $Q$ is calculated as follow:
\begin{equation}
	\begin{split}
        \forall A_k \in {\cal A},\, b_Q^+(A_k) &= 
            \sum_{i=1}^{|A_k|} (-1)^{i+1} 
            \sum_{B \in comb(A_k, i)} \left| \sublatx{B}{F_k^*} \right| - r_Q^+(B) \\
            &= \sum_{i=1}^{|A_k|} (-1)^{i+1} 
            \sum_{B \in comb(A_k, i)} 2^{|F^*_k| - i} - 1 - r_Q^+(B)
	\end{split}
\end{equation}
The total number of estimated positive bags covered by $Q$ is finally deduced, thanks to the property \ref{prop:latticeinter}, by summing $b_Q^+(A_k)$ for all $A_k \in {\cal A}$:
\begin{equation}
	B^+_Q = \sum_{A_k \in {\cal A}} b_Q^+(A_k)
\end{equation}

\noindent{\bf Example.} Let's consider a set $E = \{a, b, c, d\}$ such that $F^*(\{a\}) = F^*(\{c\}) = \{a, b, c, d\} = F^*_1$. We put both in the equivalence class $A_1 = \{ a, c\}$. Say we learned a DNF $Q$ such that $F_Q(\{a\}) = \{a, b\}$ and $F_Q(\{c\}) = \{c, d\}$.
Figure \ref{fig:equivposbags} shows, the set of positive bags engendered by $a$ whose coverage is ensured by $F_Q(\{ a \})$ (${\cal L}_a$), the set of positive bags engendered by $c$  whose coverage is ensured by $F_Q(\{ c \})$ (${\cal L}_c$) and the set of positive bags engendered by both $a$ and $c$ whose coverage is ensured by both $F_Q(\{ a \})$ and $F_Q(\{ c \})$ (${\cal L}_{ac}$).
%
%



The inclusion-exclusion principle is, again, uses to estimate the number of positive bags whose coverage is ensured by the closures of the elements of the equivalence class $A_1$. First, for all $x \in E$ the number of positive bags engendered by $x$ whose coverage is ensured by $F_Q(\{x\})$ ($b^+(x) - |ec(x)|$) are summed. As a consequence, the bag $bag^+(A_1, \{ a, c \})$ is counted two times. This mistake is suppressed by subtracting the number of positive bags whose coverage is ensured by both $F_Q(\{ a \})$ and $F_Q(\{ c \})$ \footnote{$\left| \sublatx{\{ a, c \}}{F_1^*} \right| - r_Q^+(\{ a, c \})$}. The calculation is hence:
\begin{eqnarray*}
	b_Q^+(A_1) &=&
    	\left( \left|\sublatx{\{a\}}{F_1^*} \right| - r_Q^+(\{a\}) \right) + 
    	\left( \left|\sublatx{\{c\}}{F_1^*} \right| - r_Q^+(\{a\}) \right) \\
    	&&- \left( \left|\sublatx{\{a, b, c\}}{F_1^*} \cap \sublatx{\{a, c, d\}}{F_1^*} \right| \right) \\
    &=& \left( 7 - 3 \right) + \left( 7 - 3 \right) - \left( 3 - 2 \right) \\
    &=& 7
\end{eqnarray*}

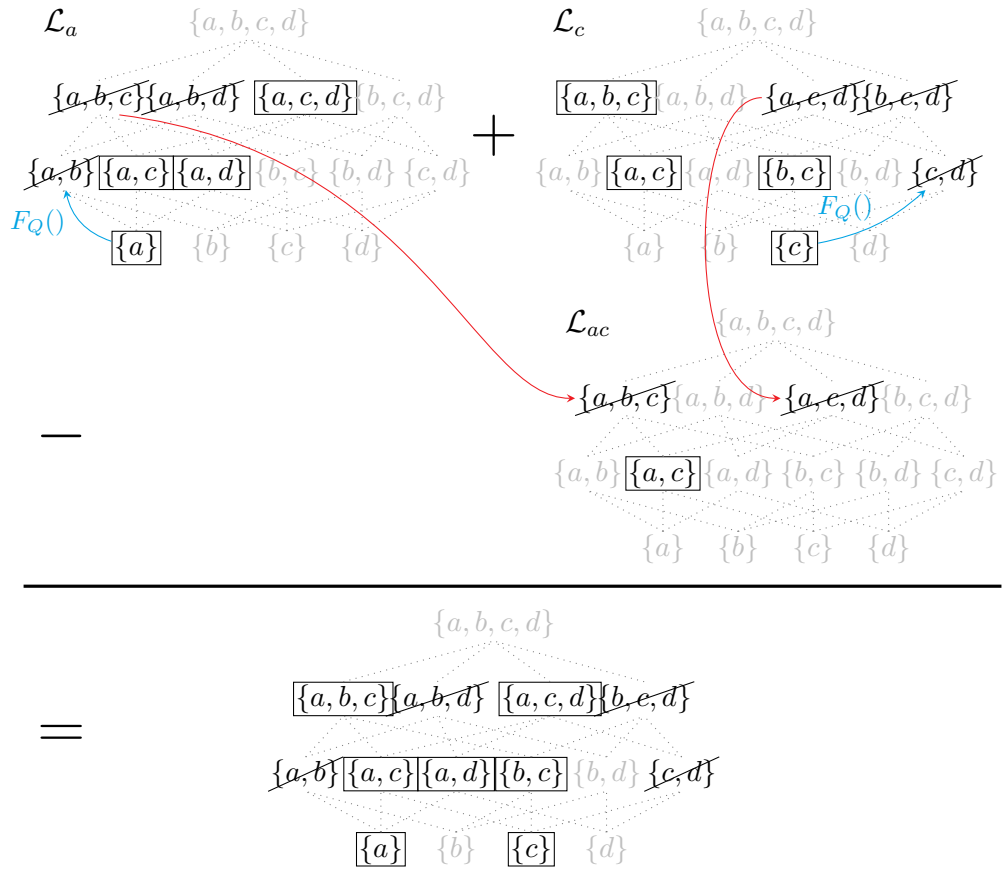
\begin{figure}
	\centering
    \begin{tikzpicture}[
        every node/.style={draw=none, inner sep=0.3mm, transform shape},
        edge/.style={thin, color=gray, dotted},
        shaded/.style={thin, color=lightgray},
        covered/.style={draw, shape=rectangle},
        rejected/.style={draw, shape=strike out},
        >=stealth,
        scale=1
    ]
        
        \begin{scope}
        	\node at (-1, 3) {\large${\cal L}_a$} ;
            \node[covered] (a) at (0, 0) {$\{ a \}$};
            \node[shaded] (b) at (1, 0) {$\{ b \}$};
            \node[shaded] (c) at (2, 0) {$\{ c \}$};
            \node[shaded] (d) at (3, 0) {$\{ d \}$};

            \node[rejected] (ab) at (-1, 1) {$\{ a, b \}$};
            \node[covered] (ac) at (0, 1) {$\{ a, c \}$};
            \node[covered] (ad) at (1, 1) {$\{ a, d \}$};
            \node[shaded] (bc) at (2, 1) {$\{ b, c \}$};
            \node[shaded] (bd) at (3, 1) {$\{ b, d \}$};
            \node[shaded] (cd) at (4, 1) {$\{ c, d \}$};

            \node[rejected] (abc) at (-0.5, 2) {$\{ a, b, c \}$};
            \node[rejected] (abd) at (0.75, 2) {$\{ a, b, d \}$};
            \node[covered] (acd) at (2.25, 2) {$\{ a, c, d \}$};
            \node[shaded] (bcd) at (3.5, 2) {$\{ b, c, d \}$};

            \node[shaded] (abcd) at (1.5, 3) {$\{ a, b, c, d \}$};

            \draw[edge] (a.north) edge (ab.south) edge (ac.south) edge (ad.south);
            \draw[edge] (b.north) edge (ab.south) edge (bc.south) edge (bd.south);
            \draw[edge] (c.north) edge (ac.south) edge (bc.south) edge (cd.south);
            \draw[edge] (d.north) edge (ad.south) edge (bd.south) edge (cd.south);

            \draw[edge] (ab.north) edge (abc.south) edge (abd.south);
            \draw[edge] (ac.north) edge (abc.south) edge (acd.south);
            \draw[edge] (ad.north) edge (abd.south) edge (acd.south);
            \draw[edge] (bc.north) edge (abc.south) edge (bcd.south);
            \draw[edge] (bd.north) edge (abd.south) edge (bcd.south);
            \draw[edge] (cd.north) edge (acd.south) edge (bcd.south);

            \draw[edge] (abc.north) edge (abcd.south);
            \draw[edge] (abd.north) edge (abcd.south);
            \draw[edge] (acd.north) edge (abcd.south);
            \draw[edge] (bcd.north) edge (abcd.south);

            \draw[Cerulean] (a) edge[->, bend left] node[midway, left, inner sep=2mm] {\small$F_Q()$} (ab) ;
            \coordinate (start_rej_a) at (abc.320) ;
        \end{scope}
        
        \begin{scope}[shift={(6.75, 0)}]
        	\node at (-1, 3) {\large${\cal L}_c$} ;
            \node[shaded] (a) at (0, 0) {$\{ a \}$};
            \node[shaded] (b) at (1, 0) {$\{ b \}$};
            \node[covered] (c) at (2, 0) {$\{ c \}$};
            \node[shaded] (d) at (3, 0) {$\{ d \}$};

            \node[shaded] (ab) at (-1, 1) {$\{ a, b \}$};
            \node[covered] (ac) at (0, 1) {$\{ a, c \}$};
            \node[shaded] (ad) at (1, 1) {$\{ a, d \}$};
            \node[covered] (bc) at (2, 1) {$\{ b, c \}$};
            \node[shaded] (bd) at (3, 1) {$\{ b, d \}$};
            \node[rejected] (cd) at (4, 1) {$\{ c, d \}$};

            \node[covered] (abc) at (-0.5, 2) {$\{ a, b, c \}$};
            \node[shaded] (abd) at (0.75, 2) {$\{ a, b, d \}$};
            \node[rejected] (acd) at (2.25, 2) {$\{ a, c, d \}$};
            \node[rejected] (bcd) at (3.5, 2) {$\{ b, c, d \}$};

            \node[shaded] (abcd) at (1.5, 3) {$\{ a, b, c, d \}$};
            
            \draw[edge] (a.north) edge (ab.south) edge (ac.south) edge (ad.south);
            \draw[edge] (b.north) edge (ab.south) edge (bc.south) edge (bd.south);
            \draw[edge] (c.north) edge (ac.south) edge (bc.south) edge (cd.south);
            \draw[edge] (d.north) edge (ad.south) edge (bd.south) edge (cd.south);

            \draw[edge] (ab.north) edge (abc.south) edge (abd.south);
            \draw[edge] (ac.north) edge (abc.south) edge (acd.south);
            \draw[edge] (ad.north) edge (abd.south) edge (acd.south);
            \draw[edge] (bc.north) edge (abc.south) edge (bcd.south);
            \draw[edge] (bd.north) edge (abd.south) edge (bcd.south);
            \draw[edge] (cd.north) edge (acd.south) edge (bcd.south);

            \draw[edge] (abc.north) edge (abcd.south);
            \draw[edge] (abd.north) edge (abcd.south);
            \draw[edge] (acd.north) edge (abcd.south);
            \draw[edge] (bcd.north) edge (abcd.south);
            
            \draw[Cerulean] (c) edge[->, bend right=15] node[midway, above left] {\small$F_Q()$} (cd) ;
            \coordinate (start_rej_c) at (acd.west) ;
        \end{scope}
        
		\begin{scope}[shift={(7, -4)}]
        	\node at (-1, 3) {\large${\cal L}_{ac}$} ;
            \node[shaded] (a) at (0, 0) {$\{ a \}$};
            \node[shaded] (b) at (1, 0) {$\{ b \}$};
            \node[shaded] (c) at (2, 0) {$\{ c \}$};
            \node[shaded] (d) at (3, 0) {$\{ d \}$};

            \node[shaded] (ab) at (-1, 1) {$\{ a, b \}$};
            \node[covered] (ac) at (0, 1) {$\{ a, c \}$};
            \node[shaded] (ad) at (1, 1) {$\{ a, d \}$};
            \node[shaded] (bc) at (2, 1) {$\{ b, c \}$};
            \node[shaded] (bd) at (3, 1) {$\{ b, d \}$};
            \node[shaded] (cd) at (4, 1) {$\{ c, d \}$};

            \node[rejected] (abc) at (-0.5, 2) {$\{ a, b, c \}$};
            \node[shaded] (abd) at (0.75, 2) {$\{ a, b, d \}$};
            \node[rejected] (acd) at (2.25, 2) {$\{ a, c, d \}$};
            \node[shaded] (bcd) at (3.5, 2) {$\{ b, c, d \}$};

            \node[shaded] (abcd) at (1.5, 3) {$\{ a, b, c, d \}$};
            
            \draw[edge] (a.north) edge (ab.south) edge (ac.south) edge (ad.south);
            \draw[edge] (b.north) edge (ab.south) edge (bc.south) edge (bd.south);
            \draw[edge] (c.north) edge (ac.south) edge (bc.south) edge (cd.south);
            \draw[edge] (d.north) edge (ad.south) edge (bd.south) edge (cd.south);

            \draw[edge] (ab.north) edge (abc.south) edge (abd.south);
            \draw[edge] (ac.north) edge (abc.south) edge (acd.south);
            \draw[edge] (ad.north) edge (abd.south) edge (acd.south);
            \draw[edge] (bc.north) edge (abc.south) edge (bcd.south);
            \draw[edge] (bd.north) edge (abd.south) edge (bcd.south);
            \draw[edge] (cd.north) edge (acd.south) edge (bcd.south);

            \draw[edge] (abc.north) edge (abcd.south);
            \draw[edge] (abd.north) edge (abcd.south);
            \draw[edge] (acd.north) edge (abcd.south);
            \draw[edge] (bcd.north) edge (abcd.south);
            
            \draw[Red, ->] (start_rej_a) .. controls +(4, -0.5) and +(-2, 0) .. (abc) ;
            \draw[Red, ->] (start_rej_c) .. controls +(-1, 0) and +(-2, 0) .. (acd) ;
        \end{scope}
        
		\begin{scope}[shift={(3.25, -8)}]
            \node[covered] (a) at (0, 0) {$\{ a \}$};
            \node[shaded] (b) at (1, 0) {$\{ b \}$};
            \node[covered] (c) at (2, 0) {$\{ c \}$};
            \node[shaded] (d) at (3, 0) {$\{ d \}$};

            \node[rejected] (ab) at (-1, 1) {$\{ a, b \}$};
            \node[covered] (ac) at (0, 1) {$\{ a, c \}$};
            \node[covered] (ad) at (1, 1) {$\{ a, d \}$};
            \node[covered] (bc) at (2, 1) {$\{ b, c \}$};
            \node[shaded] (bd) at (3, 1) {$\{ b, d \}$};
            \node[rejected] (cd) at (4, 1) {$\{ c, d \}$};

            \node[covered] (abc) at (-0.5, 2) {$\{ a, b, c \}$};
            \node[rejected] (abd) at (0.75, 2) {$\{ a, b, d \}$};
            \node[covered] (acd) at (2.25, 2) {$\{ a, c, d \}$};
            \node[rejected] (bcd) at (3.5, 2) {$\{ b, c, d \}$};

            \node[shaded] (abcd) at (1.5, 3) {$\{ a, b, c, d \}$};
            
            \draw[edge] (a.north) edge (ab.south) edge (ac.south) edge (ad.south);
            \draw[edge] (b.north) edge (ab.south) edge (bc.south) edge (bd.south);
            \draw[edge] (c.north) edge (ac.south) edge (bc.south) edge (cd.south);
            \draw[edge] (d.north) edge (ad.south) edge (bd.south) edge (cd.south);

            \draw[edge] (ab.north) edge (abc.south) edge (abd.south);
            \draw[edge] (ac.north) edge (abc.south) edge (acd.south);
            \draw[edge] (ad.north) edge (abd.south) edge (acd.south);
            \draw[edge] (bc.north) edge (abc.south) edge (bcd.south);
            \draw[edge] (bd.north) edge (abd.south) edge (bcd.south);
            \draw[edge] (cd.north) edge (acd.south) edge (bcd.south);

            \draw[edge] (abc.north) edge (abcd.south);
            \draw[edge] (abd.north) edge (abcd.south);
            \draw[edge] (acd.north) edge (abcd.south);
            \draw[edge] (bcd.north) edge (abcd.south);
        \end{scope}
        
        \node (plus) at (4.75, 1.5) {\Huge$+$} ;
        \node (minus) at (-1, -2.5) {\Huge$-$} ;
        \draw[very thick] (-1.5, -4.5) -- (11.5, -4.5) ;
        \node (minus) at (-1, -6.5) {\Huge$=$} ;
    \end{tikzpicture}
    
    \caption{Calculation of the number of positive bags covered engendered by $A_1 = \{ a, c \}$ and covered by $Q$, with $E = \{ a, b, c, d \}$, $F_1^* = E$, $F_Q(\{a\}) = \{ a, b \}$ and $F_Q(\{c\}) = \{ c, d \}$.}
    \label{fig:equivposbags}
\end{figure}

\subsubsection{Number of negative bags covered by a solution}

It is much more simpler to compute the number of covered negative bags (The false positives count). This is mainly due to their reduced number: each element $x \in E$ engenders $|E \setminus F^*(\{x\})|$ negatives bags. That is to say one bag per element that does not belong to the closure of $\{x\}$.

Given a DNF $Q$, we compute $\forall x \in E$ its elementary closure $F_Q(\{x\})$. The number of negative bags covered by $Q$ for a given element $x \in E$ is $|F_Q(\{x\}) \setminus F^*(\{x\})|$, that is to say the number of retrieved element which were not expected.
Hence, the total number of negative bags covered by a solution $Q$ is expressed by
\begin{equation}
	B^-_Q = \sum_{x \in E} |F_Q(\{x\}) \setminus F^*(\{x\})|
\end{equation}

\subsection{The intrinsic quality measure}

Based on the count of positive and negative bags covered by a DNF $Q$, we define a new quality measure taking into account both the retrieved elementary closed sets (as the extrinsic quality measure) and the internal behavior of the V-type pretopological space $(E, a_Q)$.

We designed our quality measure such that DNFs exposing a high number of covered positive bags and a low number of covered negative bags are highlighted. Because the positive and negative bags count do not have the same order of magnitude, we consider the $log_2$ of the number of covered positive bags. Our quality measure is inspired from the \emph{tozero score} defined by \citet{DBLP:conf/icml/BlockeelPS05} and is defined by
\begin{equation}
	h(Q) = \frac{log_2(B^+_Q)}{log_2(B^+_Q) + B^-_Q + p}
\end{equation}
Where $p$ is \emph{a parameter that influences how strongly the measure is pulled toward zero} \citep{DBLP:conf/icml/BlockeelPS05}.

We call it an intrinsic quality measure because it takes into account the result of the closure of any set included in the learned elementary closed sets, as opposed to the extrinsic quality measure that consider only the elementary closed sets.

\subsection{Multiple Instance LPS}

MI LPS is a multiple instance variation of LPS. It takes as input a set of target elementary closed sets $S^*$ and a collection of $k$ neighborhoods ${\cal V} = \{ V_1, \ldots, V_k \}$, and outputs a positive DNF $Q$ such that the elementary closed sets provided by the pseudo-closure operator $a_Q(.)$ fit best with to $S^*$.

MI LPS works in an iterative fashion, it builds a positive DNF by appending a new disjunctive clause to the DNF $Q_i$ built at the previous iteration. MI LPS is implemented in the same spirit as Greedy LPS, they mostly differ from the quality measure they are based on. It chooses the best clause based on its intrinsic quality measure instead of the extrinsic one. Thus we do not show any pseudo-code of MI LPS because it would be the same as Greedy LPS (algorithm \ref{algo:greedy}) but with the extrinsic measure replaced by the intrinsic measure. Since the intrinsic quality measure is based on the number of covered bags (by a potential solution $Q$), it is by itself the reason why MI LPS belongs to the MI framework.


    
        
        	
            
            

    
    
    \section{Experiments}\label{sec:expe}

We build a method to learn a propagation model based on the so called intrinsic measure, as an opposition to the extrinsic measure used in the previous original contributions. We assume that, because the intrinsic measure takes into account propagations on subsets of element (and not only on singletons) it captures the quality and the potential of a model under construction and is more suitable than the extrinsic measure. In order to validate the interestingness of the intrinsic measure and the efficiency of the associated MI approach, we propose an experimental validation. This experiment consists in learning a propagation phenomenon, that models a forest fire. On these simulations proposed, our aim is to evaluate how each LPS approach is able to recognize the (true) underlying pretopological model.

We consider a rectangular grid containing a set of cells $E$. We also consider a collection of eight neighborhoods ${\cal V} = \{ V_1, \ldots, V_8\}$ corresponding to the eights Moore's neighborhoods showed in Figure \ref{fig:moore}. Given a known propagation model $a_{Q^*}(.)$, we build the (V-type) pretopological space $(E, a_{Q^*})$. Our goal is to learn a positive DNF $Q$ such that $a_{Q^*}(.)$ and $a_{Q}(.)$ output similar closed sets. 

We built three different models wisely named according to our estimation of the difficulty to learn them:
\begin{itemize}
	\item $Q^*_{Simple} = q_4 \vee q_6 \vee q_7$
	\item $Q^*_{Medium} = (q_4 \wedge q_6) \vee (q_5 \wedge q_8) \vee q_7$
	\item $Q^*_{Hard} = q_3 \vee q_5 \vee (q_2 \wedge q_4) \vee (q_4 \wedge q_7) \vee (q_6 \wedge q_7 \wedge q_8)$
\end{itemize}
We then derive three pseudo-closure operators ($a_{Q^*_{Simple}}$, $a_{Q^*_{Medium}}$ and $a_{Q^*_{Hard}}$) to learn from these DNF (${Q^*_{Simple}}$, ${Q^*_{Medium}}$ and ${Q^*_{Hard}}$ respectively). Each pseudo-closure operator is used to model the propagation of a forest fire based on a percolation process \citep{DBLP:conf/rivf/AmorLL07}. Each DNF models a different propagation process influenced by the wind. For example, $Q^*_{Simple}$ models a forest fire under the influence of a south-west wind: a cell $x \in E$ burns when a cell located at its north, east or north-east is burning, so the fire propagates toward the south-west corner of the grid. Figure \ref{fig:forestfire} shows the propagations obtained with $Q^*_{Simple}$, $Q^*_{Medium}$ and $Q^*_{Hard}$, we clearly observe the fires modeled by $Q^*_{Simple}$ propagating toward the south-west corner.
%
%

For each of these models, we built seven 15x15 grids ${\cal G} = \{ G_0, G_{10}, G_{20}, G_{30}, G_{40}, G_{50}, G_{60} \}$ filled with a growing percentage of obstructed cells. An obstructed cell is a cell $x \in E$ such that $\forall V_i \in {\cal V}, V_i(x) = \{x\}$, that is to say a cell that cannot be expanded by a pseudo-closure operator. So $G_0$ contains 0\% of obstructed cells, $G_{10}$ contains 10\% of obstructed cells, and so on. We should mention that the obstructed cells in $G_{i+1}$ are those in $G_i$ plus 10\% more. Then, for each grid $G_i \in {\cal G}$, we simulated the propagation of a fire ignited in each cell $x \in G_i$. The resulting set of burnt cells was then assigned to the value of $F_{Q^*}(\{x\})$. Some of those simulations are showed in Figure \ref{fig:forestfire}. The green cells are combustible cells, the gray ones are obstacles (i.e. obstructed cells). the salmon cell is the starting point of the fire and the red cells are the burnt cells.

\begin{figure}
    \centering
    \begin{tikzpicture}[scale=0.8]
        \foreach \x in {1, 2, 3}
        \foreach \y in {1, 2, 3}{
        \fill[orange] (\x, \y) rectangle (\x + 1, \y + 1) ;
        }
        \fill[pink] (2, 2) rectangle (3, 3) ;
        \draw[thick] (0, 0) grid (5, 5) ;

        \node[draw=none] at (1.5, 3.5) {$V_1$} ;
        \node[draw=none] at (1.5, 2.5) {$V_2$} ;
        \node[draw=none] at (1.5, 1.5) {$V_3$} ;
        \node[draw=none] at (2.5, 3.5) {$V_4$} ;
        \node[draw=none] at (2.5, 1.5) {$V_5$} ;
        \node[draw=none] at (3.5, 3.5) {$V_6$} ;
        \node[draw=none] at (3.5, 2.5) {$V_7$} ;
        \node[draw=none] at (3.5, 1.5) {$V_8$} ;
    \end{tikzpicture}

    \caption{Moore's neighborhoods}
    \label{fig:moore}
\end{figure}
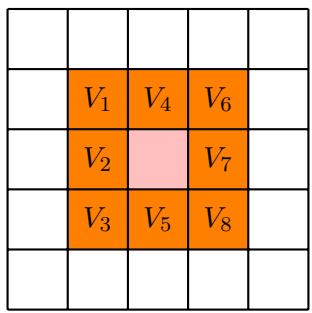

\begin{figure}
    \centering
	\begin{tikzpicture}
		\matrix[matrix of nodes] (m) {
            \includegraphics[width=3.5cm]{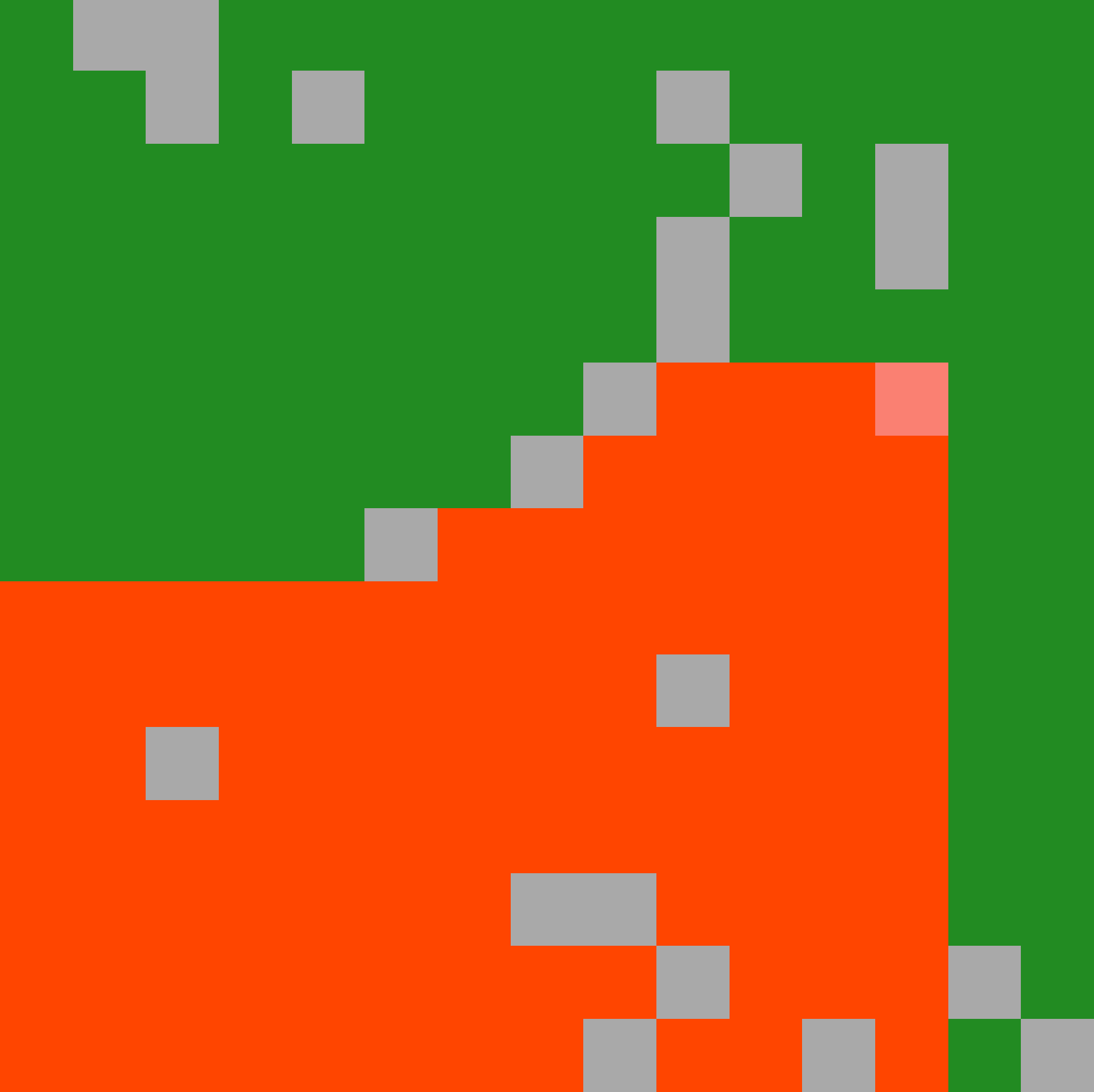} & 
            \includegraphics[width=3.5cm]{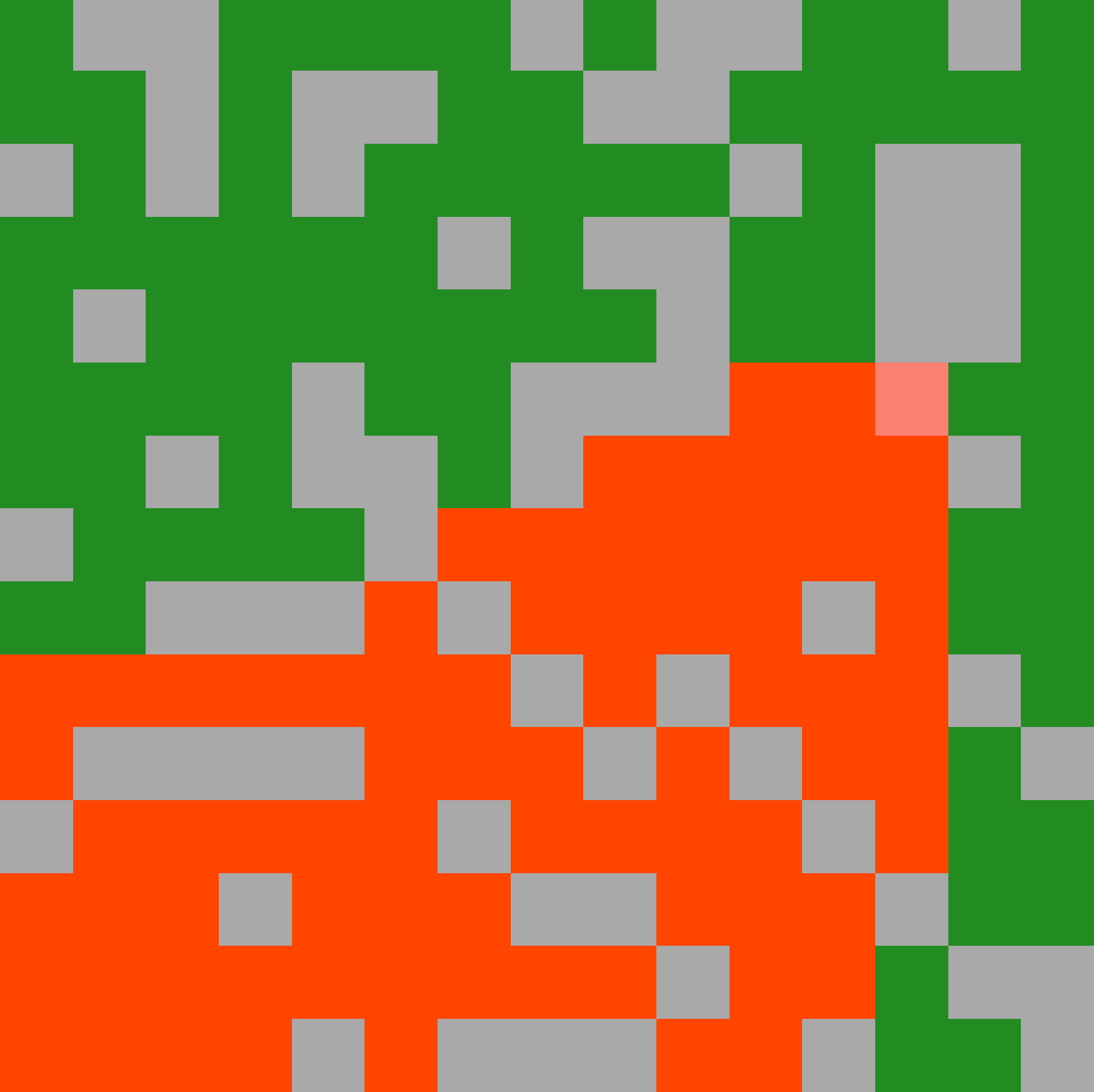} & 
            \includegraphics[width=3.5cm]{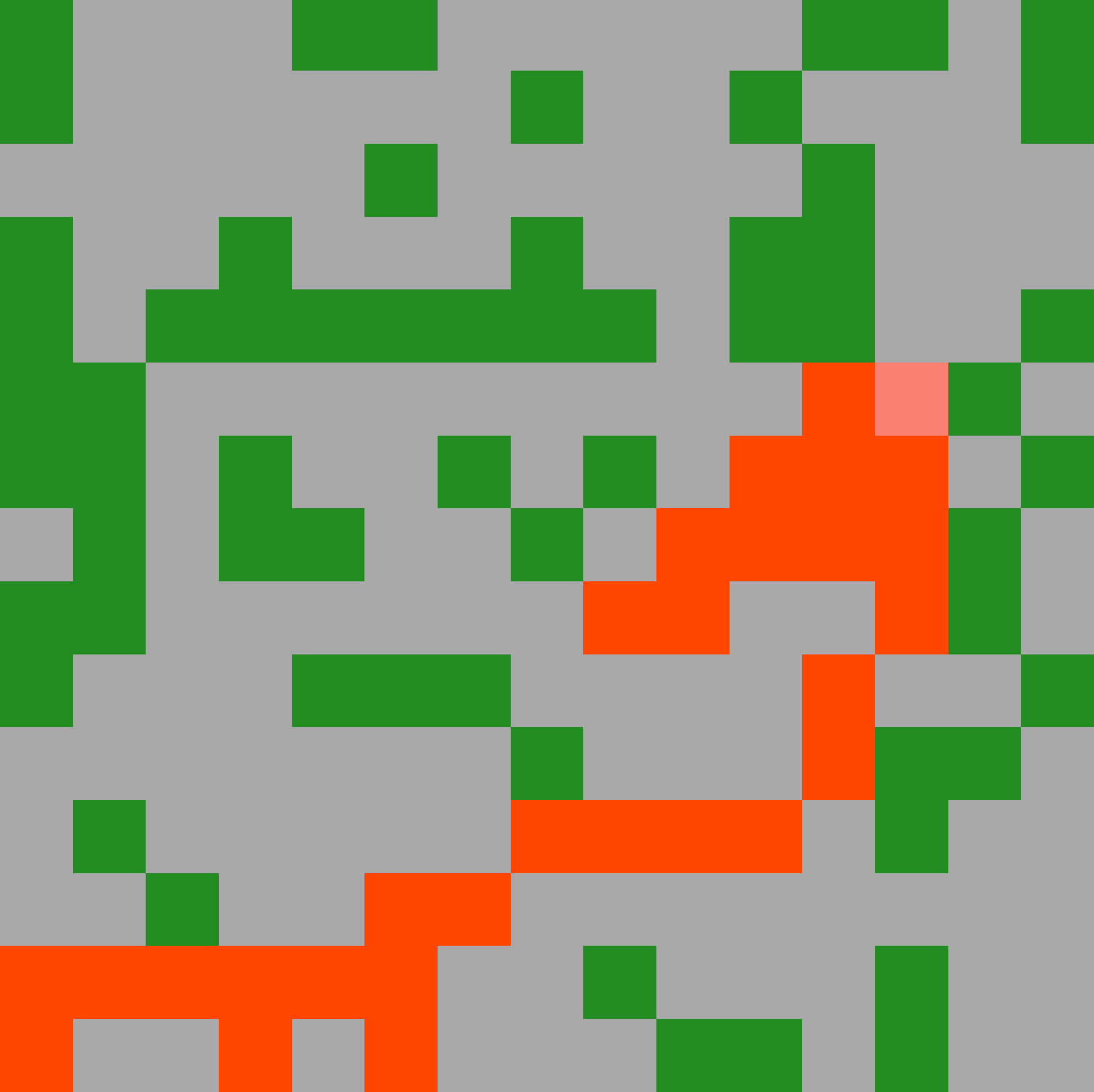} \\

            \includegraphics[width=3.5cm]{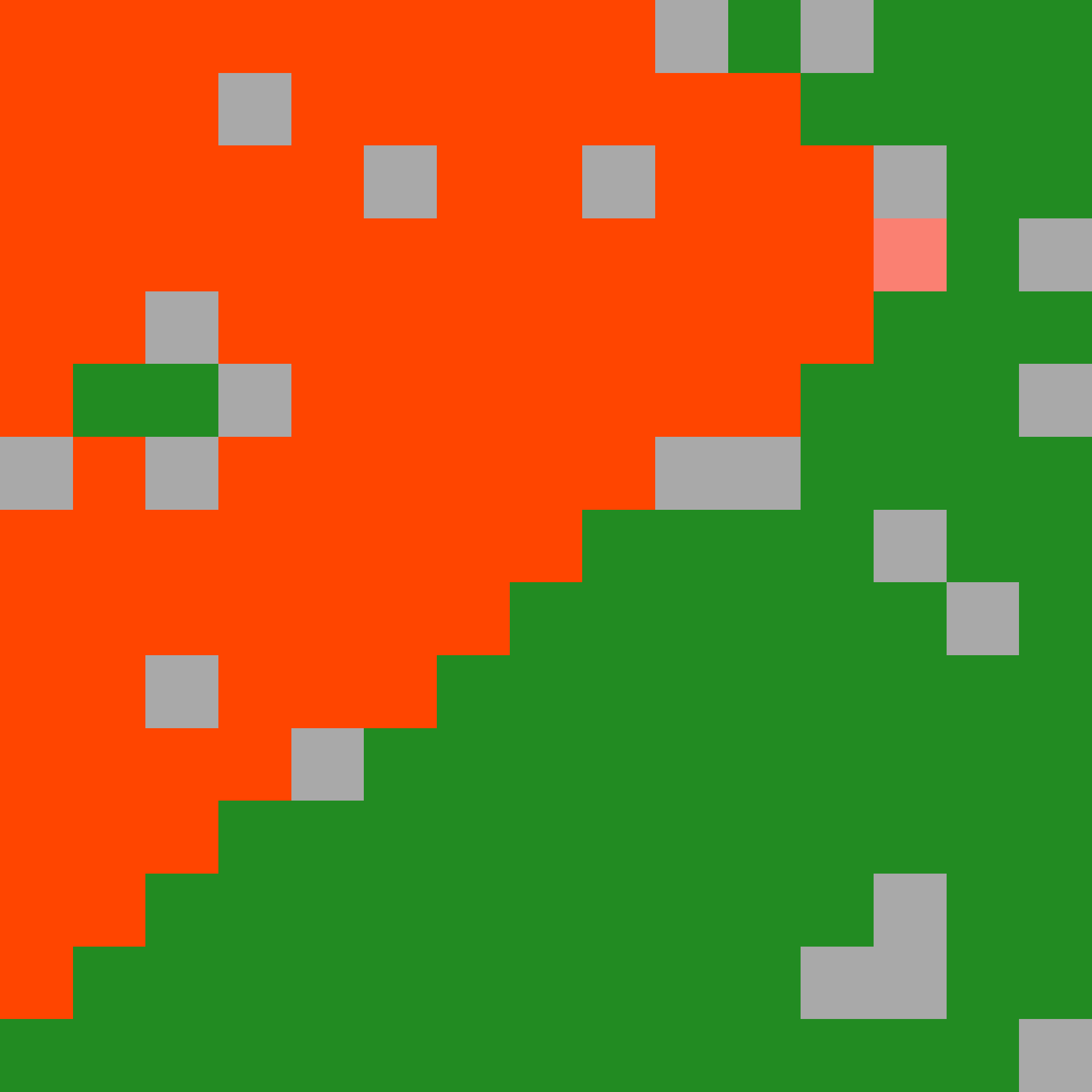} & 
            \includegraphics[width=3.5cm]{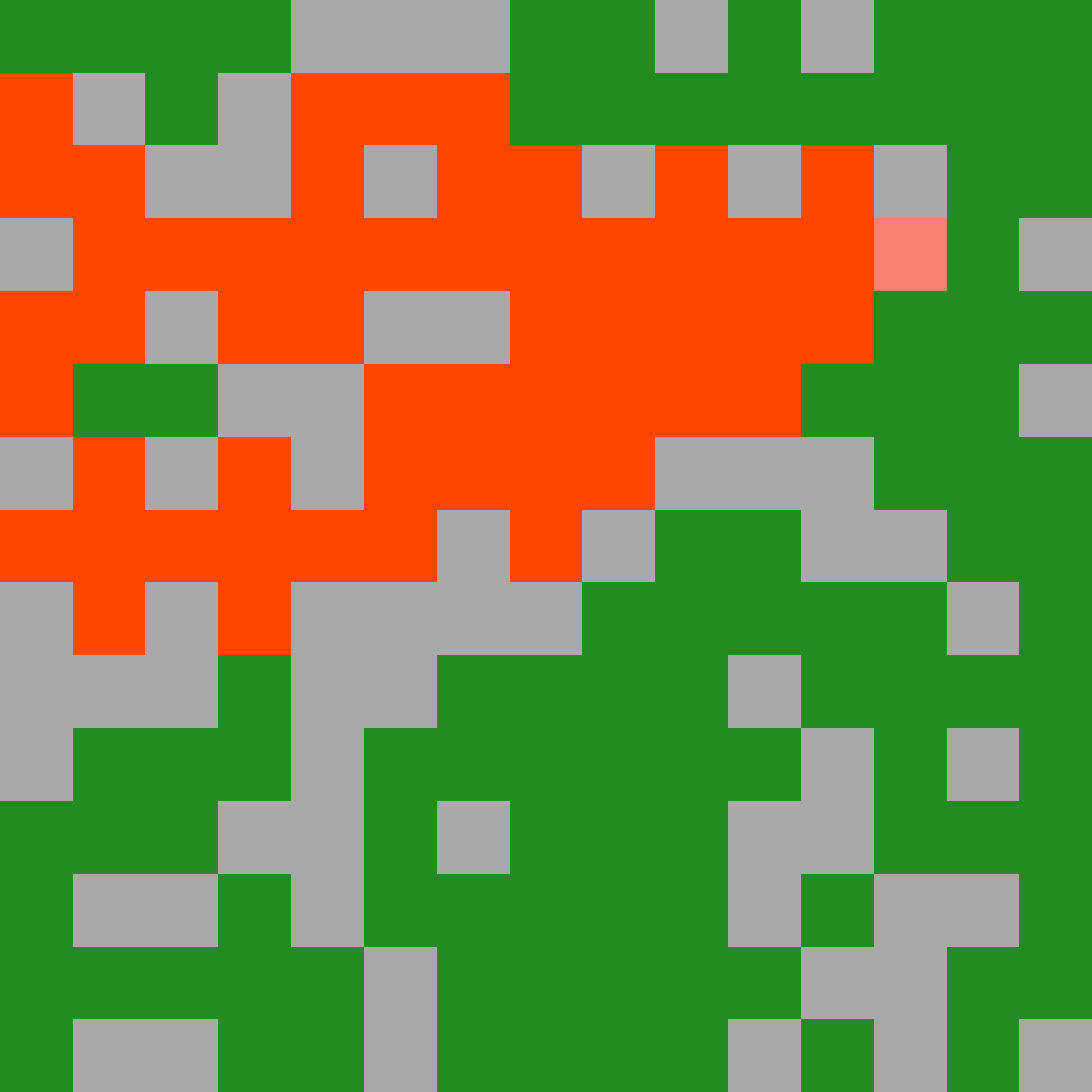} & 
            \includegraphics[width=3.5cm]{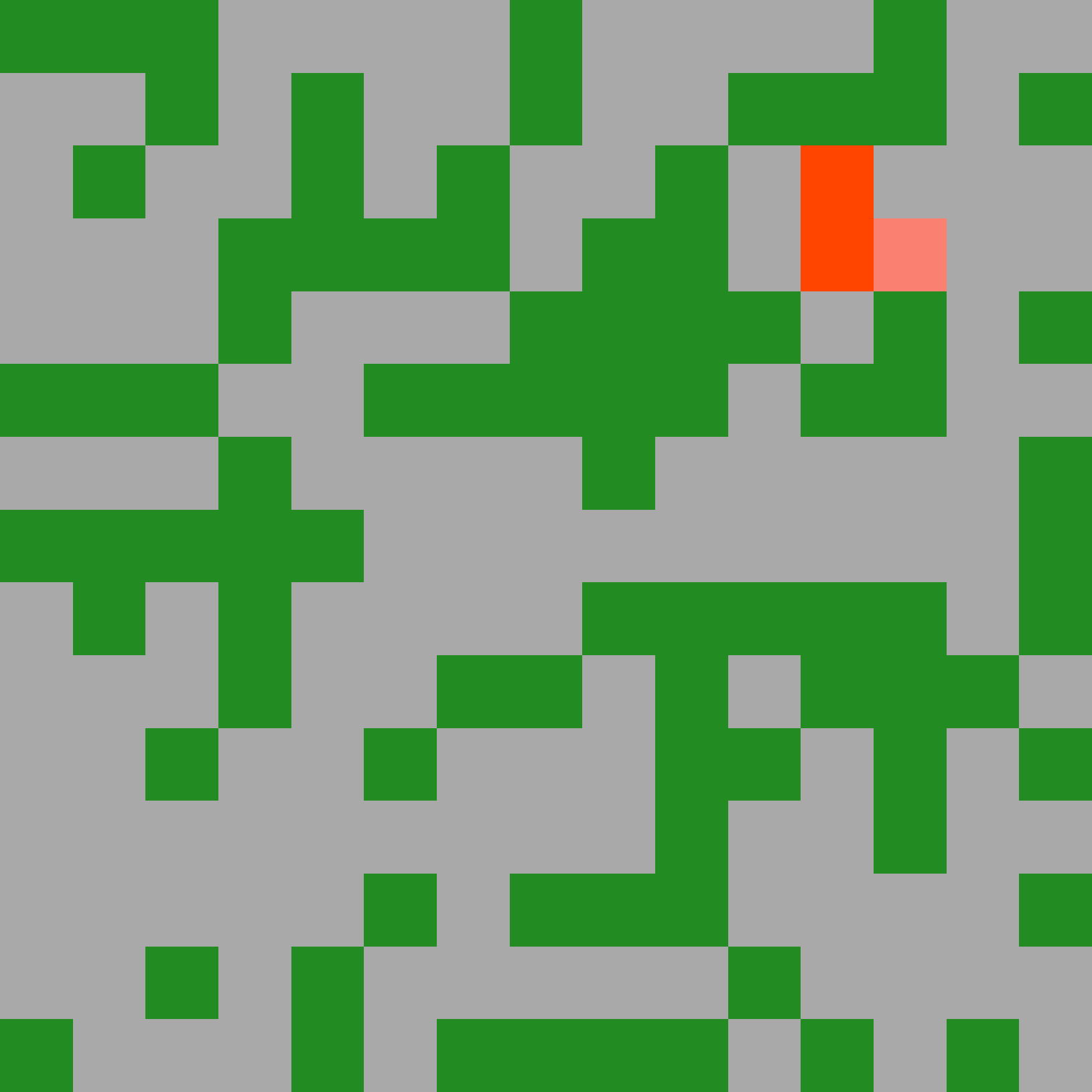} \\

            \includegraphics[width=3.5cm]{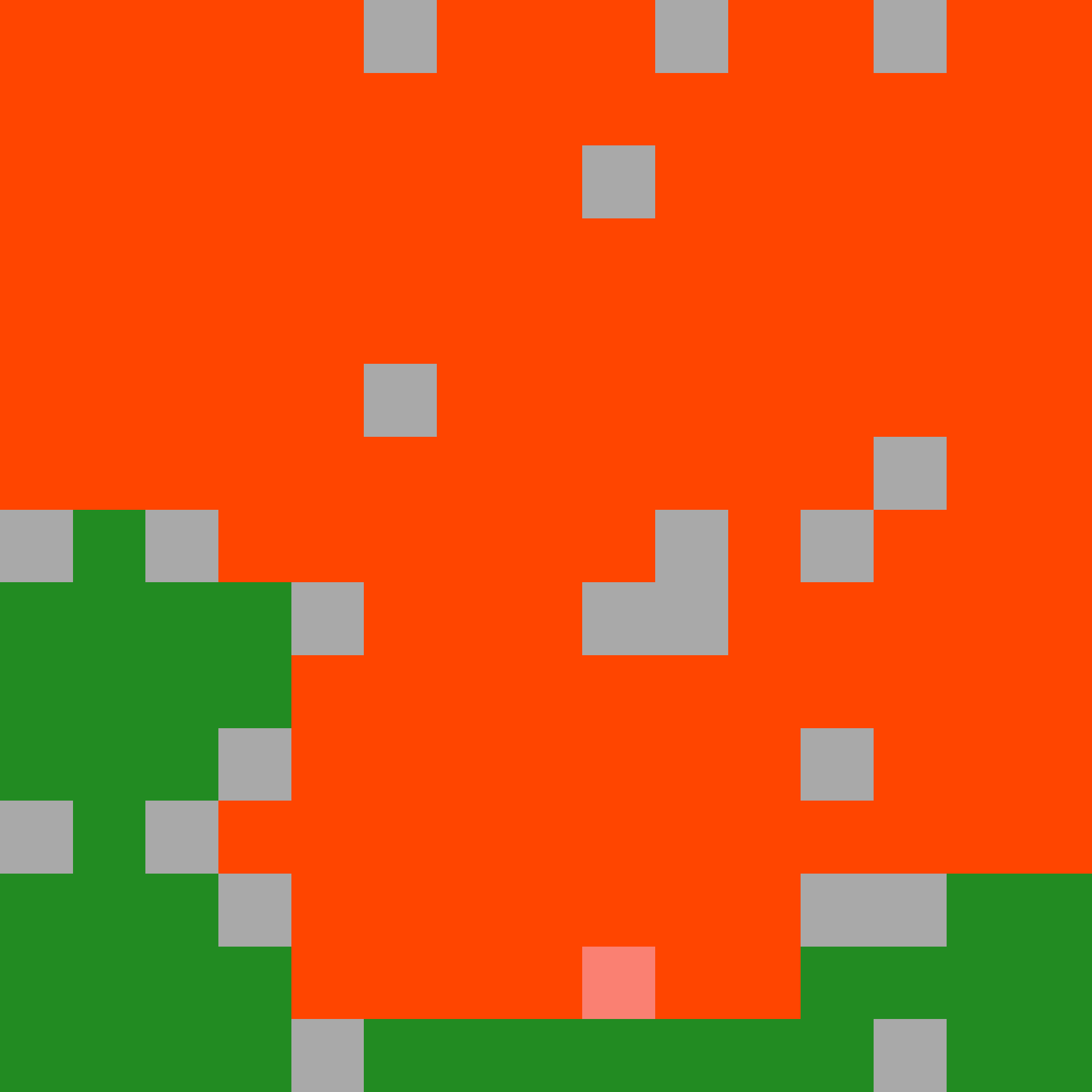} & 
            \includegraphics[width=3.5cm]{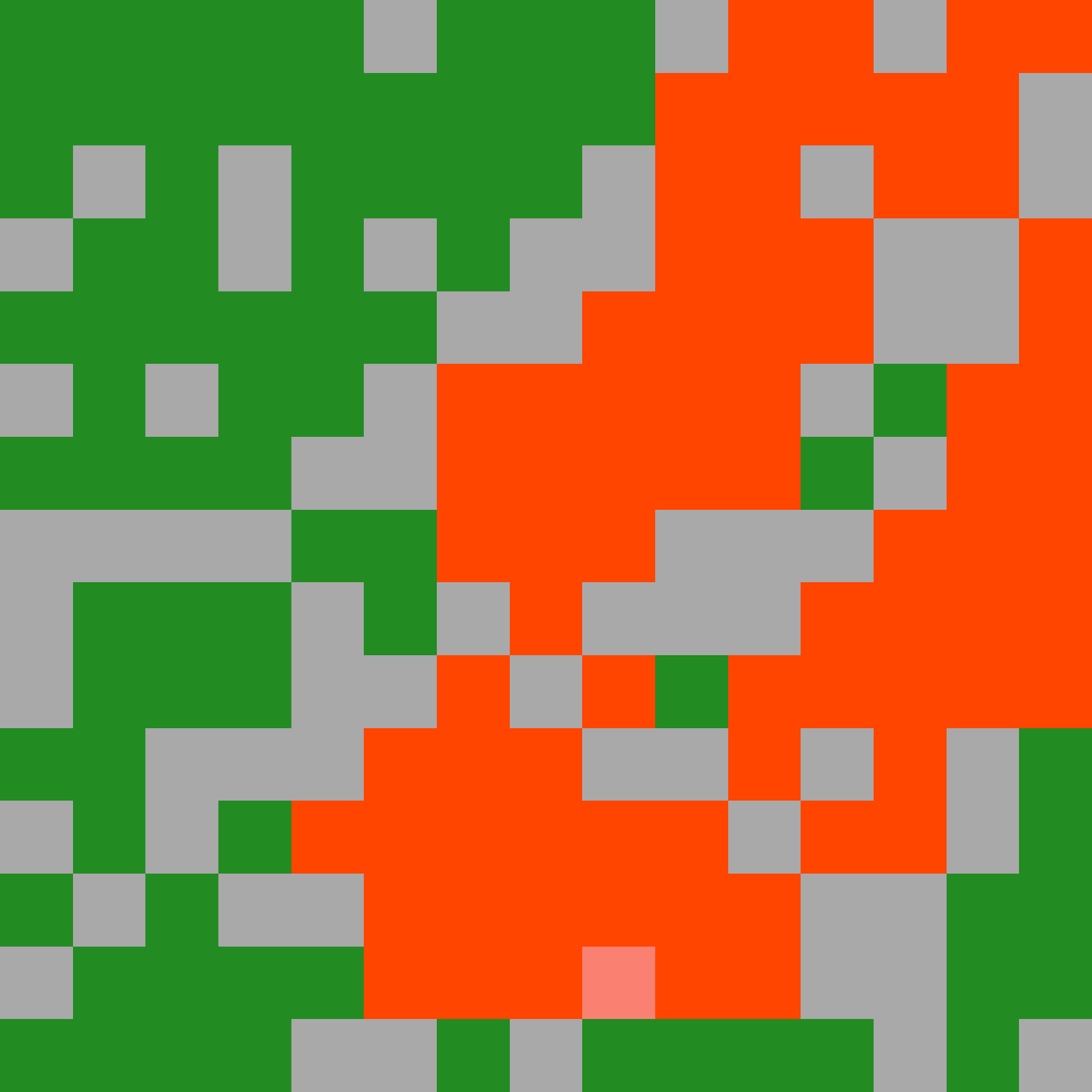} & 
            \includegraphics[width=3.5cm]{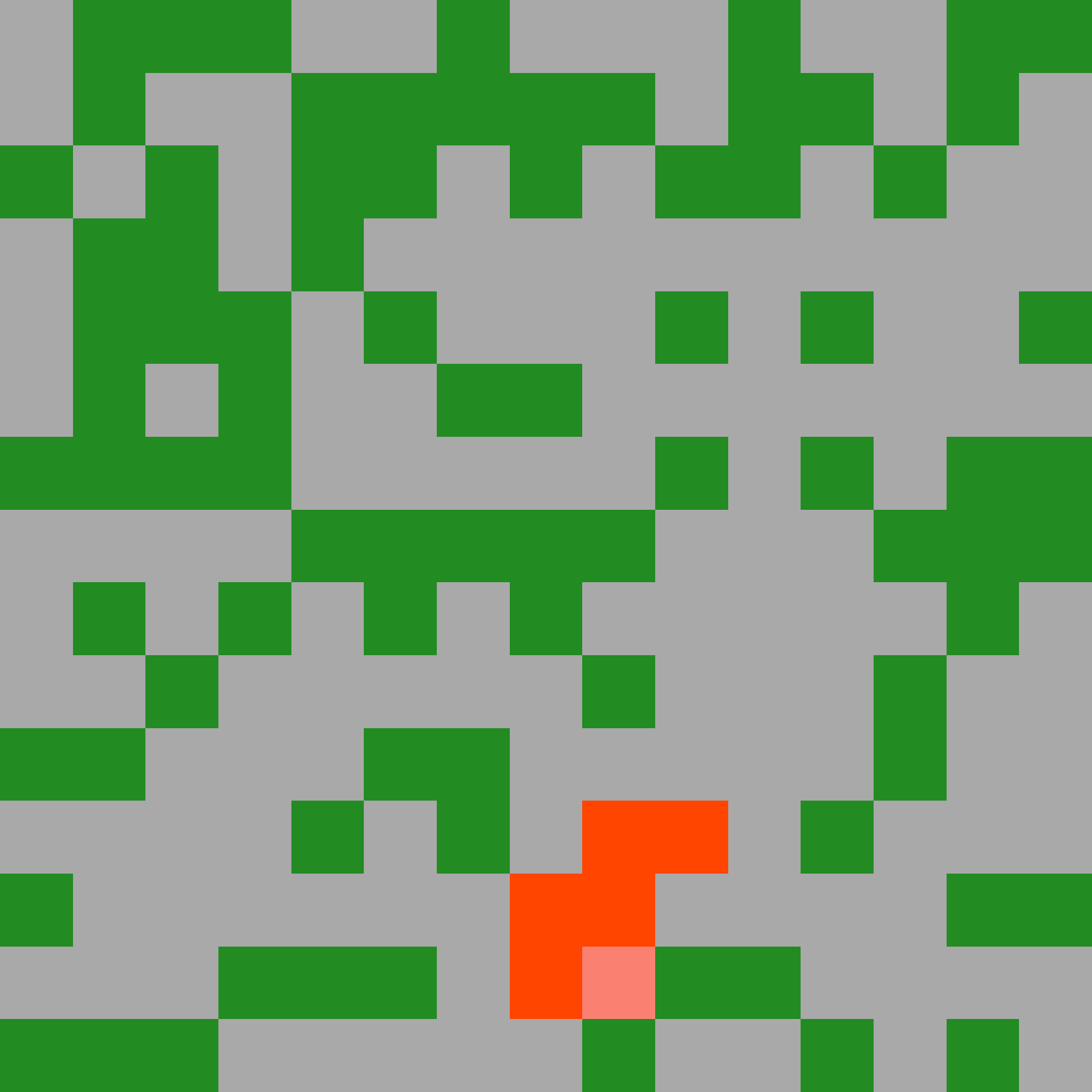} \\
        } ;
        
        \matrix[below=0.1cm of m] (legend) {
            \node[anchor=east,fill=treegreen, minimum width=0.15cm, minimum height=0.15cm] (trees_square) {} ; &
            \node[anchor=west] (trees_text) {\small Trees} ; &[1em]

            \node[anchor=east,fill=firered, minimum width=0.15cm, minimum height=0.15cm] (fire_square) {} ; &
            \node[anchor=west] (fire_text) {\small Fire} ; &[1em]

            \node[anchor=east, fill=obstaclegray, minimum width=0.15cm, minimum height=0.15cm] (obstacles_square) {} ; &
            \node[anchor=west] (obstacles_text) {\small Obstacles} ; &[1em]

            \node[anchor=east, fill=emitterpink, minimum width=0.15cm, minimum height=0.15cm] (emitter_square) {} ; &
            \node[anchor=west] (emitter_text) {\small Origin} ; \\
		} ;
	\end{tikzpicture}
    
    \caption{From left to right: 10\%, 30\% and 60\% obstacles, from top to bottom $Q_{Simple}$, $Q_{Medium}$ and $Q_{Hard}$}
    \label{fig:forestfire}
\end{figure}

For each experiment we picked 30\% of the inflammable cells and compute their propagation with the proposed model in order to build the set of target elementary closed sets $S^*$. Then $S^*$ and $\cal V$ were feed to the input of each of the LPS approaches: Genetic LPS (numerical and logical formalisms), Greedy LPS and MI LPS in order to compare their results. Each algorithm was tested with two sets of parameters: both Genetic LPS were tested with an initial population size of 100 and 500 ; Greedy LPS and MI LPS were tested with a beam search size of 1 and 5.
Lastly, each experiment was realized ten times, so for each $Q \in \{ Q^*_{simple}, Q^*_{medium}, Q^*_{hard} \}$, we built ten times a collection of seven grids. Figure \ref{fig:plots} shows the average results of each experiment: the F-measure obtained is indicated on the ordinate axis and the execution time is proportional to the size of each point (the smaller the better).


The first result is to notice that LPS MI obtains the best score most of the time. We also remark that setting a higher beam size results in better scores in exchange of an insignificant increase of the execution time: with a beam search of size 5, MI LPS is able to retrieve the whole target solution without errors (its F-measure equals 1)!
Greedy LPS is the second best approach in term of F-measure and offers quite similar performance. However, the scores obtained by Greedy LPS does not seem to be influenced by the size of its beam search, so unlike LPS MI, we may not hope to obtain better results by simply enlarging the beam.
The quality of the output of Logical Genetic LPS is highly sensible to the size of the initial population. Thus we may expect excellent results if its population is large enough. But its execution time is also extremely sensible to this parameter, in some cases, it took hours to obtain a result whereas it was a matter of seconds for both Greedy and MI LPS. 
Finally, Numerical Genetic LPS performed the worst in terms of both F-measure and execution time. The only difference between Numerical Genetic LPS and Logical Genetic LPS is the model they learned, the first one learns a weight vector while the last one learns a positive DNF. So we impute its poor performance to the unadapted weighting vector approach.

This experiments illustrates and confirms clearly the motivation of the present contribution: first, a positive DNF is more appropriate than a numerical vector to model an expansion process ; then, it shows that while genetic algorithms have proven their worth, their prerequisites (a huge initial population) are too important in our context ; Finally, we compare two greedy algorithms: MI LPS and Greedy LPS which both learn a positive DNF but in different contexts. MI LPS is a multiple instance algorithm while Greedy LPS is a classical supervised algorithm; they use different quality measures: MI LPS choosing the clause maximizing an intrinsic measure while Greedy LPS using an extrinsic measure. This experiment shows that the MI approach is able to retrieve a target model more easily than a more classical approach, whatever the complexity of the model to learn.

    

\begin{figure}
	\centering
    \begin{subfigure}[b]{\textwidth}
		\includegraphics[width=\textwidth]{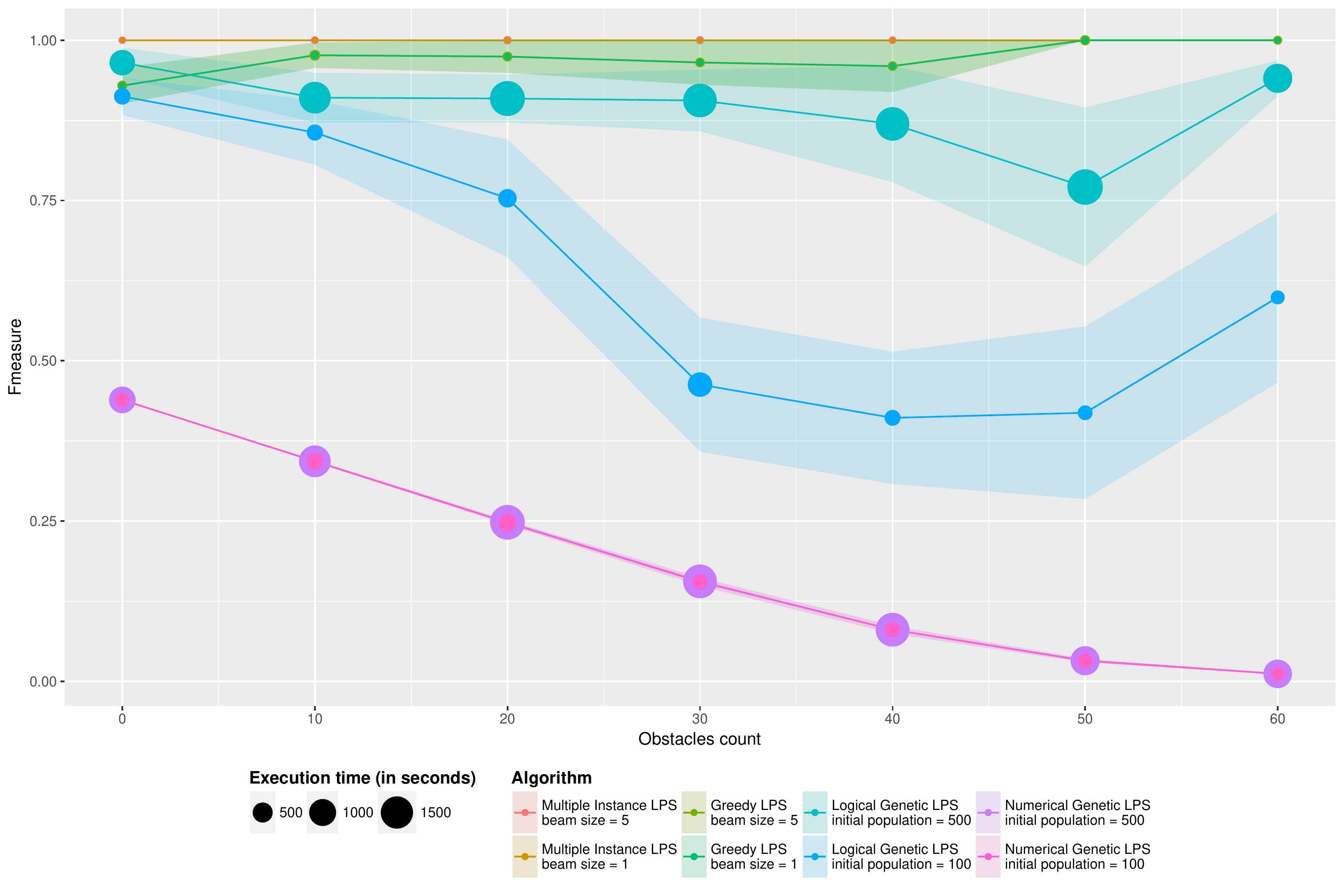}
        \subcaption{$Q_{Simple}$}
	\end{subfigure}
    \begin{subfigure}[b]{\textwidth}
		\includegraphics[width=\textwidth]{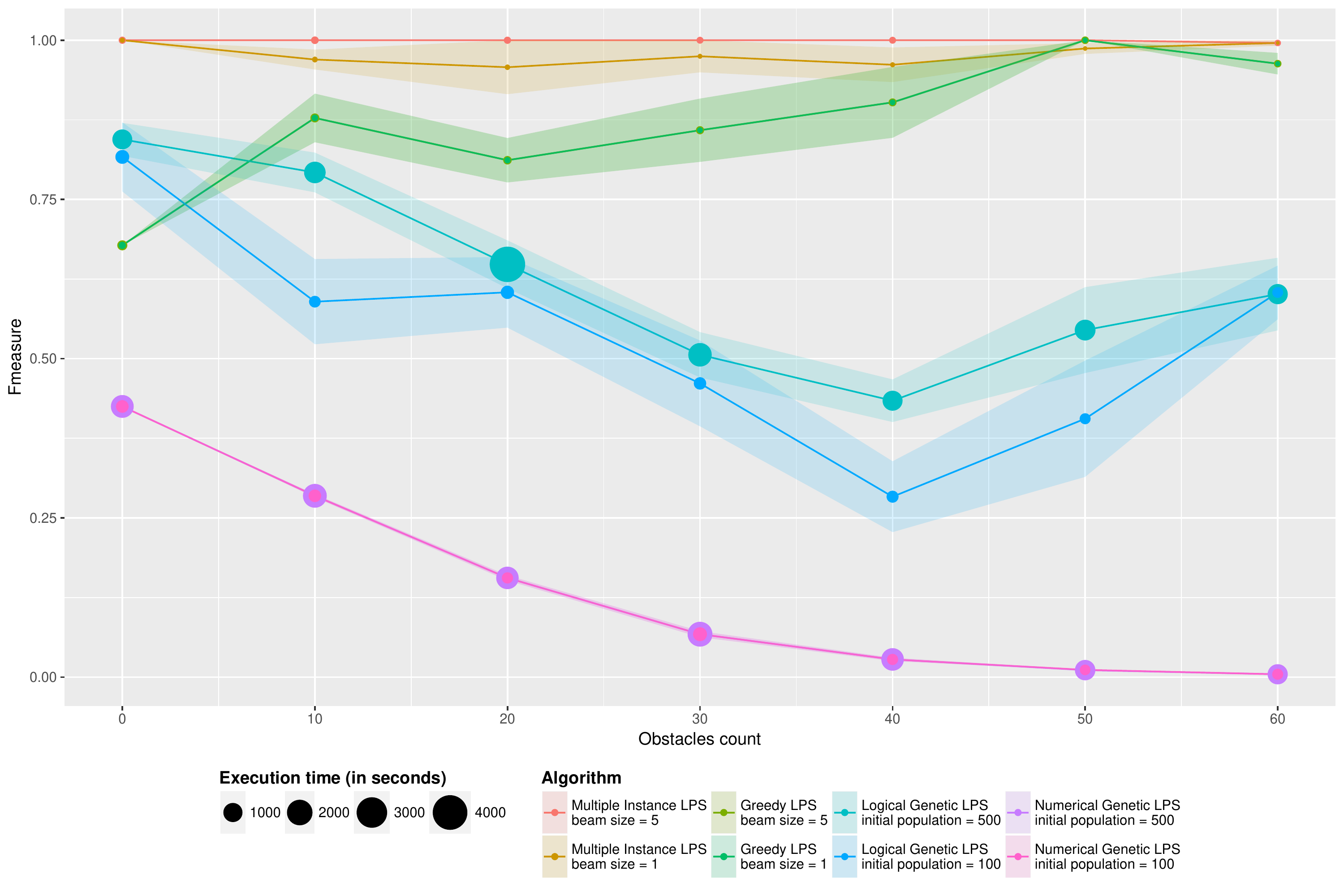}
        \subcaption{$Q_{Medium}$}
	\end{subfigure}
    
    \caption{Mean output's quality of the LPS algorithms with different settings.}
\end{figure}
\begin{figure}\ContinuedFloat
    \begin{subfigure}[b]{\textwidth}
		\includegraphics[width=\textwidth]{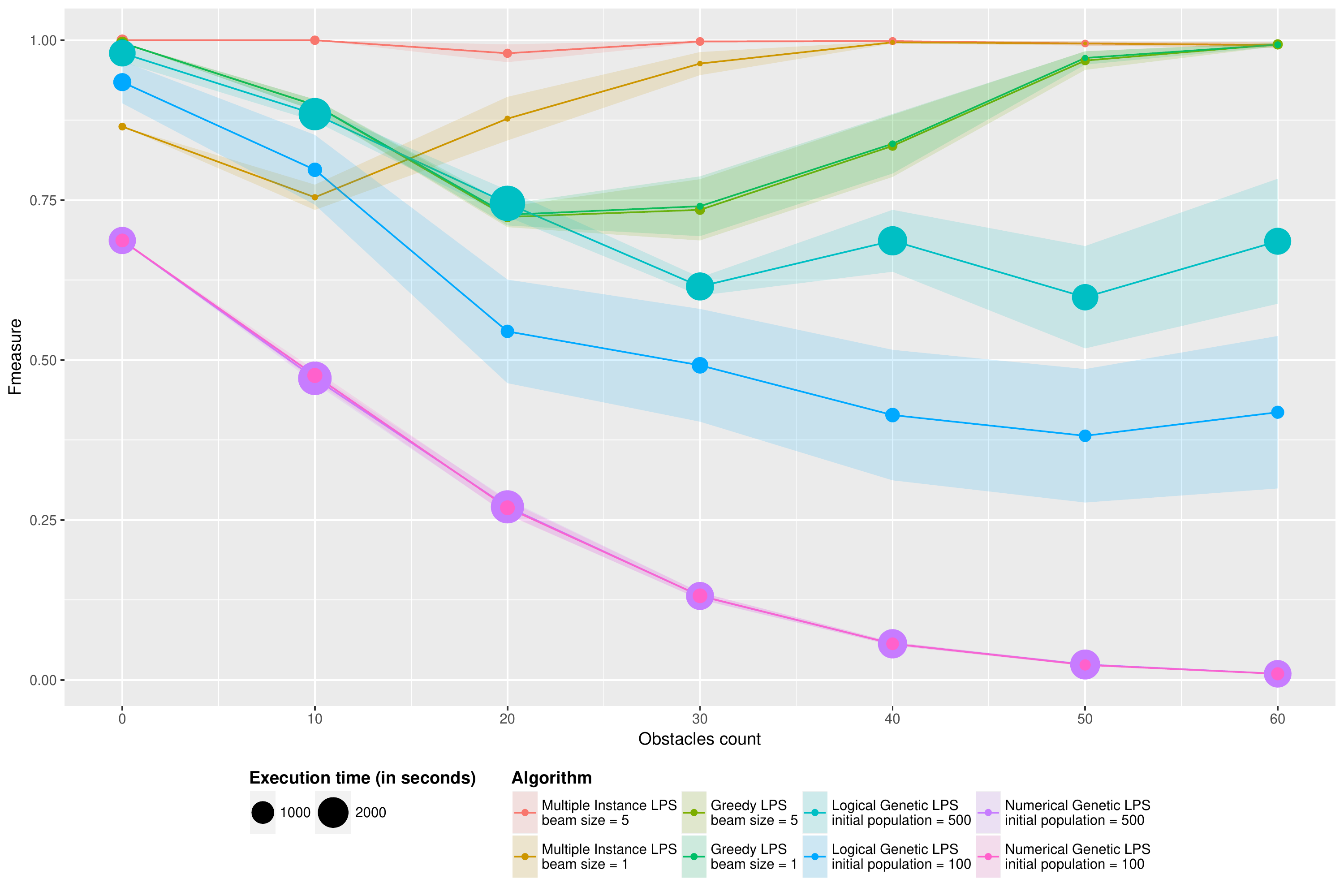}
        \subcaption{$Q_{Hard}$}
	\end{subfigure}
   
    \caption{Mean output's quality of the LPS algorithms with different setting.}
    \label{fig:plots}
\end{figure}

    \section{Conclusion}

We introduce a new method to learn (in a supervised context) a pretopological space based on its elementary closed sets: the Multi-Instance LPS (MI LPS). This method is the continuation of the work done by \citet{cleuziou2015learning} where a pseudo-closure was modeled with a numerical vector and learned using a genetic algorithm. Due to the lack of expressiveness of the numerical formalism, we switch to a more appropriate logical modeling using a positive DNF deriving in a pseudo-closure operator. This representation of a pretopological space allows to consider many more solutions while restraining the hypothesis space, thus making its exploration easier and faster. LPS MI is also faster than its competitors to deliver its output since its learning strategy explores efficiently the search space due to an appropriate (intrinsic) quality measure.

Learning a pretopological space is a wonderful improvement of how we leverage the pretopology theory. In previous works, the pseudo-closure operator was hand fixed and based on a collection of neighborhoods that \emph{seems} to model accurately a given problem/phenomenon. Learning a pseudo-closure operator allows us to build a good combination among any collection of neighborhoods, such that only the useful neighborhoods are considered in the learned pseudo-closure operator.

Future works will investigate other models of pseudo-closure operator. For example, a pseudo-closure operator based on first logic order formulas could lead to even better expressiveness. We also envisage to learn, in the same time a combination of neighborhoods and a set of parameters defining these neighborhoods (from a proximity matrix) in hopes of reducing the false positive rate of the learned solution.

We also plan to refine our estimation of the positive/negative bags covered by a solution since a better approximation will permit to learn more accurate solutions. We will also apply our algorithm on a broader range of application cases such as the modeling of expansion behaviors in networks (social networks, biological networks, road networks\ldots), for instance, we could learn the model underneath the expansion of a community in a social network. 
And, as it was the problem which motivated the learning of a pretopological space, we will use our method to learn models useful for text mining, such as task of building lexical taxonomies.

    \bibliographystyle{elsarticle-num-names}
    \bibliography{references}
\end{document}